\documentclass{ieeeaccess}
\usepackage{cite}
\usepackage{amsmath,amssymb,amsfonts}
\usepackage{algorithmic}
\usepackage{graphicx}
\usepackage{textcomp}

\usepackage{times}
\usepackage{epsfig}
\usepackage{graphicx}
\usepackage{amsmath}
\usepackage{amssymb}
\usepackage{multirow}
\usepackage{amsthm}
\usepackage[ruled,vlined]{algorithm2e}
\usepackage{subcaption}
\usepackage{comment}
\usepackage[export]{adjustbox}

\newtheorem{proposition}{Proposition}
\newtheorem{theorem}{Theorem}
\newtheorem{definition}{Definition}
\newtheorem*{proposition1}{Proposition~\ref{proposition1}}
\newtheorem*{theorem1}{Theorem~\ref{theorem1}}
\newtheorem*{theorem2}{Theorem~\ref{theorem2}}
\newcommand{\norm}[1]{\left\lVert#1\right\rVert}

\def\BibTeX{{\rm B\kern-.05em{\sc i\kern-.025em b}\kern-.08em
    T\kern-.1667em\lower.7ex\hbox{E}\kern-.125emX}}
\begin{document}
\doi{10.1109/ACCESS.2022.3171741}

\title{Fisher Task Distance and Its Application in Neural Architecture Search}
\author{\uppercase{Cat P. Le},
\uppercase{Mohammadreza Soltani, Juncheng Dong, Vahid Tarokh}.}
\address{Department of Electrical and Computer Engineering, Duke University, Durham, NC 27708 USA }
\tfootnote{This work was supported in part by the Army Research Office grant No. W911NF-15-1-0479.}

\markboth
{Cat P. Le \headeretal: Fisher Task Distance and Its Application in Neural Architecture Search}
{Cat P. Le \headeretal: Fisher Task Distance and Its Application in Neural Architecture Search}

\corresp{Corresponding author: Cat P. Le (e-mail: cat.le@duke.edu).}

\begin{abstract}
We formulate an asymmetric (or non-commutative) distance between tasks based on Fisher Information Matrices, called Fisher task distance. This distance represents the complexity of transferring the knowledge from one task to another. We provide a proof of consistency for our distance through theorems and experiments on various classification tasks from MNIST, CIFAR-10, CIFAR-100, ImageNet, and Taskonomy datasets. Next, we construct an online neural architecture search framework using the Fisher task distance, in which we have access to the past learned tasks. By using the Fisher task distance, we can identify the closest learned tasks to the target task, and utilize the knowledge learned from these related tasks for the target task. Here, we show how the proposed distance between a target task and a set of learned tasks can be used to reduce the neural architecture search space for the target task. The complexity reduction in search space for task-specific architectures is achieved by building on the optimized architectures for similar tasks instead of doing a full search and without using this side information. Experimental results for tasks in MNIST, CIFAR-10, CIFAR-100, ImageNet datasets demonstrate the efficacy of the proposed approach and its improvements, in terms of the performance and the number of parameters, over other gradient-based search methods, such as ENAS, DARTS, PC-DARTS. 
\end{abstract}


\titlepgskip=-15pt

\maketitle

\section{Introduction}
\label{sec:introduction}
\PARstart{T}{his} paper is motivated by a common assumption made in transfer and lifelong learning~\cite{silver2008guest,  finn2016deep, mihalkova2007mapping, niculescu2007inductive, luolabel, Razavian:2014:CFO:2679599.2679731, 5288526, DBLP:journals/corr/abs-1801-06519, DBLP:journals/corr/FernandoBBZHRPW17, DBLP:journals/corr/RusuRDSKKPH16, zamir2018taskonomy}: similar tasks usually have similar neural architectures and shares common knowledge. In multi-task learning ~\cite{standley2020tasks}, it also found that utilizing the shared knowledge of related tasks can boost the performance of training. However, if the non-related tasks were trained together, the overall performance degrades significantly. Up until now, in order to identify the closeness of tasks, people often depend on the domain knowledge, or brute-force approach with transfer learning. Building on this observation, we propose a novel task metric that is easy and efficient to compute, and represents the complexity of transferring the knowledge of one task to another, called \emph{Fisher Task} distance (FTD). This distance is a non-commutative measure by design, since transferring the knowledge of a comprehensive task to a simple one is much easier than the other way around. FTD is defined in terms of the Fisher Information matrix defined as the second-derivative of the loss function with respect to the parameters of the model under consideration. By definition, FTD is always greater or equal to zero, where the equality holds if and only if it is the distance from a task to itself. To show that our task distance is mathematically well-defined, we provide several theoretical analysis. Moreover, we empirically verify that the FTD is statistically a consistent distance through experiments on numerous tasks and datasets, such as MNIST, CIFAR-10, CIFAR-100, ImageNet, and Taskonomy~\cite{zamir2018taskonomy}. In particular, Taskonomy dataset indicate our computational efficiency while achieving similar results in term of distance between tasks as the brute-force approach proposed by~\cite{zamir2018taskonomy}. Next, we instantiate our proposed task distance on Neural Architecture Search (NAS). In the traditional NAS~\cite{liu2018darts,liu2018progressive, liu2017hierarchical, zoph2016neural, 9091879}, the search for the architecture of the target task often starts from scratch, with some prior knowledge about the dataset to initialize the search space. The past learned tasks are not considered as the useful knowledge for the target task. Being motivated by the advantages of transfer learning, we would like to apply the knowledge from previous learned tasks to the target task to remove the dependency on the domain knowledge. Here, we construct a continuous NAS framework using our proposed task distance, that is capable of learning an appropriate architecture for a target task based on its similarity to other learned tasks. For a target task, the closest task in a given set of baseline tasks is identified and its corresponding architecture is used to construct a neural search space for the target task without requiring prior domain knowledge. Subsequently, our gradient-based search algorithm called FUSE~\cite{le2021task} is applied to discover an appropriate architecture for the target task. Briefly, the utilization of the related tasks' architectures helps reduce the dependency on prior domain knowledge, consequently reducing the search time for the final architecture and improving the robustness of the search algorithm. Extensive experimental results for the classification tasks on MNIST~\cite{lecun2010mnist}, CIFAR-10, CIFAR-100~\cite{krizhevsky2009learning}, and ImageNet~\cite{ILSVRC15} datasets demonstrate the efficacy and superiority of our proposed approach compared to the state-of-the-art approaches. 

In this paper, we provide the relevant works in transfer learning and neural architecture search literature in Section~\ref{related}. Next, we introduce the definition of tasks, the Fisher task distance, and theoretical analysis in Section~\ref{approach}. The continuous NAS framework, called Task-aware Neural Architecture Search, is proposed in Section~\ref{sec:nas}. Lastly, we provide the experimental studies of our distance and its application in NAS in Section~\ref{sec:experiment}.


\section{Related Works} \label{related}

The task similarity has been mainly considered in the transfer learning (TL) literature. Similar tasks are expected to have similar architectures as manifested by the success of applying transfer learning in many applications~\cite{silver2008guest,  finn2016deep, mihalkova2007mapping, niculescu2007inductive, luolabel, Razavian:2014:CFO:2679599.2679731, 5288526, DBLP:journals/corr/abs-1801-06519, DBLP:journals/corr/FernandoBBZHRPW17, DBLP:journals/corr/RusuRDSKKPH16, zamir2018taskonomy}. 
However, the main goal in TL is to transfer trained weights from a related task to a target task. Recently, a measure of closeness of tasks based on the Fisher Information matrix has been used as a regularization technique in transfer learning~\cite{chen2018coupled} and continual learning~\cite{kirkpatrick2017overcoming} to prevent catastrophic forgetting. Additionally, the task similarity has also been investigated between visual tasks in~\cite{zamir2018taskonomy, pal2019zeroshot, dwivedi2019, achille2019task2vec, wang2019neural, pmlr-v119-standley20a}. These works only focus on weight-transfer and do not utilize task similarities for discovering the high-performing architectures. Moreover, the introduced measures of similarity from these works are often assumed to be symmetric which is not typically a realistic assumption. For instance, consider learning a binary classification between cat and dog images in the CIFAR-10 dataset. It is easier to learn this binary task from a pre-trained model on the entire CIFAR-10 images with 10 classes than learning the 10-class classification using the knowledge of the binary classification in CIFAR-10. 

In the context of the neural architecture search (NAS), the similarity between tasks has not been explicitly considered. Most NAS methods focus on reducing the complexity of search by using an explicit architecture search domain and specific properties of the given task at hand. NAS techniques have been shown to offer competitive or even better performance to those of hand-crafted architectures. In general, these techniques include approaches based on evolutionary algorithms~\cite{real2018regularized, 9336721}, reinforcement learning~\cite{zoph2016neural, 9091879},  and optimization-based approaches~\cite{liu2018progressive}. However, most of these NAS methods are computationally intense and require thousands of GPU-days operations. To overcome the computational issues and to accelerate the search procedure, recently, differentiable search methods~\cite{cai2018proxylessnas,liu2018darts,noy2019asap,luo2018neural,xie2018snas, wan2020fbnetv2, awad2020differential, he2020milenas,8681706} have been proposed. These methods, together with random search methods~\cite{li2018massively,li2020random,sciuto2019evaluating} and sampling sub-networks from one-shot super-networks~\cite{zoph2018learning, bender2018understanding, li2020random,cho2019one, yang2020cars}, can significantly speed up the search time in the neural architecture space. Additionally, some reinforcement learning methods with weight-sharing~\cite{pham2018efficient, DBLP:journals/corr/abs-2101-07415, bender2020can}, similarity architecture search~\cite{le2021task, Nguyen2020_cbc}, neural tangent kernel~\cite{chen2021neural}, network transformations~\cite{cai2018efficient,elsken2018efficient, jin2018efficient, hu2019efficient}, and few-shot approaches~\cite{cai2019once, zhou2020econas, zhang2020overcoming, zhao2020few} have yielded time-efficient NAS methods. None of these approaches consider the task similarities in their search space.
In contrast, our approach exploits asymmetric relation between tasks to reduce the search space and accelerate the search procedure.

\section{Fisher Task Distance} \label{approach}

Before discussing the task distance, we recall the definition of the Fisher Information matrix for a neural network. 

\begin{definition}[Fisher Information Matrix]
Let $N$ be a neural network with data $X$, weights $\theta$, and the negative log-likelihood loss function $L(\theta) := L(\theta,X)$. The Fisher Information Matrix is defined as follows:
\begin{align}\label{Fihermatrix}
    F(\theta) =\mathbb{E}\Big[\nabla_{\theta} L(\theta)\nabla_{\theta} L(\theta)^T\Big] = -\mathbb{E}\Big[\mathbf{H}\big(L(\theta)\big)\Big],
\end{align}
where $\mathbf{H}$ is the Hessian matrix, i.e., $\mathbf{H}\big(L(\theta)\big)= \nabla_{\theta}^2L(\theta)$, and expectation is taken w.r.t the distribution of data.
\end{definition}

We use the empirical Fisher Information Matrix computed as follows:
\begin{align}\label{emprical_fisher}
    \hat{F}(\theta) = \frac{1}{|X|}\sum_{i\in X} \nabla_{\theta} L^i(\theta)\nabla_{\theta} L^i(\theta)^T,
\end{align}
where $L^i(\theta)$ is the loss value for the $i$\textsuperscript{th} data point in $X$~\footnote{We use $F$ instead of $\hat{F}$ onward for the notation simplicity.}.

Consider a dataset $X=X^{(1)}\cup X^{(2)}$ with $X^{(1)}$ and $X^{(2)}$ denote the training and the test data, respectively. Let's denote a task $T$ and its corresponding dataset $X$ jointly by a pair $(T, X)$. Also, let $\mathcal{P}_{N}((T,X^{(2)}))\in [0,1]$ be a function that measures the performance of a given architecture $N$ on a task $T$ using the test dataset $X^{(2)}$. 

\begin{definition}[$\varepsilon$-approximation Network for Task $T$]
An architecture $N$ is called an $\varepsilon$-approximation network for task $T$ and the corresponding data $X$ if it is trained using training data $X^{(1)}$ such that $\mathcal{P}_{N}(T,X^{(2)}) \geq 1 - \varepsilon$, for a given $0 < \varepsilon < 1$.
\end{definition}

In practice, architectures for $\varepsilon$-approximation networks for a given task $T$ are selected from a pool of well-known hand-designed architectures. 

\begin{algorithm}[t]
\SetKwInput{KwInput}{Input}         
\SetKwInput{KwOutput}{Output}       
\SetKwFunction{TaskDistance}{TaskDistance}
\DontPrintSemicolon

\KwData{$X_a = \{X_a^{(1)} \cup X_a^{(2)}\}, X_b = \{X_b^{(1)} \cup X_b^{(2)}\}$}
\KwInput{$\varepsilon$-approx. network $N$}
\KwOutput{Distance from task $a$ to task $b$}

    \SetKwProg{Fn}{Function}{:}{}
    \Fn{\TaskDistance{$X_a^{(1)},X_b, N$}}{
        Initialize $N_a, N_b$ from N\;
        Train $N_a$ using $X_a^{(1)}$, $N_b$ using $X_b^{(1)}$\;
        Compute $F_{a,b}$ (equation~\ref{emprical_fisher}) using $X_b^{(2)}$ on $N_a$\;
        Compute $F_{b,b}$ (equation~\ref{emprical_fisher}) using $X_b^{(2)}$ on $N_b$\;
        
        \KwRet $\displaystyle d[a, b] = \frac{1}{\sqrt{2}} \norm{F_{a,b}^{1/2} - F_{b,b}^{1/2}}_F$\;
    }
    
\caption{Fisher Task Distance}
\label{alg1}
\end{algorithm}

\begin{definition}[Fisher Task Distance]\label{frechdist}
Let $a$ and $b$ be two tasks with $N_a$ and $N_b$ denote their corresponding $\varepsilon$-approximation networks, respectively. Let $F_{a, b}$ be the Fisher Information Matrix of $N_a$ with the dataset $X_b^{(2)}$ from the task $b$, and $F_{b, b}$ be the Fisher Information Matrix of $N_b$ with the dataset $X_b^{(2)}$ from the task $b$. We define the FTD from the task $a$ to the task $b$ based on Fr\'echet distance as follows:
\begin{align}\label{frechet_org}
    d[a, b] = \frac{1}{\sqrt{2}} \textbf{Tr} \Big(F_{a,b} + F_{b,b} - 2(F_{a,b}F_{b,b})^{1/2} \Big)^{{1/2}},
\end{align}
where $\textbf{Tr}$ denotes the trace of a matrix.
\end{definition}
In this paper, we use the diagonal approximation of the Fisher Information matrix since computing the full Fisher matrix is prohibitive in the huge space of neural network parameters. We also normalize them to have a unit trace. As a result, the FTD in~\eqref{frechet_org} can be simplified as follows:
\begin{align}\label{distance}
    d[a, b] &= \frac{1}{\sqrt{2}} \norm{F_{a,b}^{1/2} - F_{b,b}^{1/2}}_F\nonumber \\
    & =\frac{1}{\sqrt{2}} \bigg[ \sum_i \Big( (F^{ii}_{a,b})^{1/2} - (F^{ii}_{b,b})^{1/2} \Big)^2 \bigg]^{1/2},
\end{align}
where $F^{ii}$ is the $i$\textsuperscript{th} diagonal entry of the Fisher Information matrix. The procedure to compute the FTD is given by Algorithm~\ref{alg1}. The FTD ranges from $0$ to $1$, with the distance $d=0$ denotes a perfect similarity and the distance $d=1$ indicates a perfect dissimilarity. As equation (\ref{distance}) shows, the FTD is \emph{asymmetric}. This aligns with human's common sense that it is often easier to transfer knowledge of a complex task to a simple task than vice versa. Note that the FTD depends on the choice of the $\varepsilon$-approximation networks. That is, by using different network architectures to represent the tasks, the computed task distance can be different. This is similar to the human being's perception: Two people can provide different values of the distance between two tasks. However, it is not likely that their perceptions are different.
For instance, people can give different values on the similarity between cat and tiger, but they agree that both cat and tiger are much different from car or plane.
In the first part of our experiments, we empirically show this intuition and illustrate that although the computed distances may be different due to the choice of $\varepsilon$-approximation networks, the trend of these distances remains consistent across different architectures. 
\begin{definition}[Structurally-Similar $\varepsilon$-approximation Networks w.r.t. $(T,X)$]\label{SSappNetTX}
Two $\varepsilon$-approximation networks $N_1$ and $N_2$ are called structurally-similar  w.r.t. $(T,X)$ if they have exact architecture (the same number of units, the same number of layers, etc), and they are trained on task $T$ using the training dataset $X^{(1)}$.
\end{definition}

Next, we present some theoretical justification for our measure of task similarity. All the proofs are provided in the appendix. Firstly, if we train any pair of structurally-similar $\varepsilon$-approximation networks w.r.t some target $(T,X)$ with the same conditions (e.g., initialization, batch order), the FTD between this pair of networks using the test dataset $X^{(2)}$ is zero. Formally, we have the following proposition:

\begin{proposition}\label{proposition1}
Let $X$ be the dataset for the target task $T$. For any pair of structurally-similar $\varepsilon$-approximation networks w.r.t $(T,X)$ using the full or stochastic gradient descent algorithm with the same initialization settings, learning rate, and the same order of data batches in each epoch for the SGD algorithm, the Fisher task distance between the above pair of $\varepsilon$-approximation networks is always zero.
\end{proposition}


In this proposition, all the training settings were assumed to be the same for two structurally-similar $\varepsilon$-approximation networks w.r.t $(T,X)$. However, an important question is whether the FTD is still a \textit{well-defined} measure regardless of the initial settings, learning rate, and the order of data batches. That is, if we train two structurally-similar $\varepsilon$-approximation networks w.r.t $(T,X)$ using SGD with different settings, will the FTD between $N_1$ and $N_2$, as defined in Equation~\eqref{distance}, be (close) zero?  We answer this question affirmatively assuming a strongly convex loss function. To this end, we invoke Polyak theorem ~\cite{Polyak1992AccelerationOS} on the convergence of the average sequence of estimation in different epochs from the SGD algorithm. While the loss function in a deep neural network is not a strongly convex function, establishing the fact that the FTD is mathematically well-defined even for this case is an important step towards the more general case in a deep neural network and a justification for the success of our empirical study. In addition, there are some recent works that try to establish Polyak theorem ~\cite{Polyak1992AccelerationOS} for the convex or even some non-convex functions in an (non)-asymptotic way~\cite{gadat2017optimal}. Here, we rely only on the asymptotic version of the theorem proposed originally by ~\cite{Polyak1992AccelerationOS}. We first recall the definition of the strongly convex function. 


\begin{definition}[Strongly Convex Function]
A differentiable function $f:\mathbb{R}^n\rightarrow\mathbb{R}$ is strongly convex if for all $x,y\in\mathbb{R}^n$ and some $\mu>0$, $f$ satisfies the following inequality:
\begin{equation}
    f(y) \geq f(x) +\nabla (f)^T(y-x)+\mu||y-x||_2^2.
\end{equation}
\end{definition}

Through this paper, we denote $\ell_{\infty}$-norm of a matrix $B$ as $||B||_{\infty}= \max_{i,j}|B_{ij}|$. Also, $|S|$ means the size of a set $S$. 

\begin{theorem}\label{theorem1}
Let $X$ be the dataset for the target task $T$. Consider $N_1$ and $N_2$ as two structurally-similar $\varepsilon$-approximation networks w.r.t. $(T,X)$ respectively with the set of weights $\theta_1$ and $\theta_2$ trained using the SGD algorithm where a diminishing learning rate is used for updating weights. Assume that the loss function $L$ for the task $T$ is strongly convex, and its 3rd-order continuous derivative exists and bounded. Let the noisy gradient function in training $N_1$ and $N_2$ networks using SGD algorithm be given by:
\begin{equation}
    g({\theta_i}_t, {\epsilon_i}_t) = \nabla L({\theta_i}_t) + {\epsilon_i}_t, \ \ for\ \ i=1,2,
\end{equation}
where ${\theta_i}_t$ is the estimation of the weights for network $N_i$ at time $t$, and $\nabla L({\theta_i}_{t})$ is the true gradient at ${\theta_i}_t$. Assume that ${\epsilon_i}_t$  satisfies $\mathbb{E}[{\epsilon_i}_t|{\epsilon_i}_0,...,{\epsilon_i}_{t-1}] = 0$, and satisfies $\displaystyle s = \lim_{t\xrightarrow[]{}\infty} \big|\big|[{\epsilon_i}_t {{\epsilon_i}_t}^T | {\epsilon_i}_0,\dots,{\epsilon_i}_{t-1}]\big|\big|_{\infty}<\infty$ almost surely (a.s.). Then the Fisher task distance between $N_1$ and $N_2$ computed on the average of estimated weights up to the current time $t$ converges to zero as $t \rightarrow \infty$. That is,
\begin{align}
    d_t = \frac{1}{\sqrt{2}}\Big|\Big|\Bar{F_1}_t^{1/2} - \Bar{F_2}_t^{1/2}\Big|\Big|_F \xrightarrow[]{\mathcal{D}}0,
\end{align}
where $\Bar{F}_{i_{t}} = F(\bar{\theta}_{i_{t}})$ with $\bar{\theta}_{i_{t}} = \frac{1}{t}\sum_t \theta_{i_{t}}$, for $i=1,2$.
\end{theorem}

In our experiments, we found out that the weights averages $\bar{\theta}_{i_{t}}, \, i=1,2$ can be replaced with only their best estimates for $\theta_{i_{t}},\, i=1,2$, respectively. Next, we show that the FTD is also a \textit{well-defined} measure between two task-data set pairs. In other words, the (asymmetric) FTD from the task $(T_A,X_A)$ to the task $(T_B,X_B)$ approaches a constant value regardless of the initialization, learning rate, and the order of data batches in the SGD algorithm provided that $X_A$ and $X_B$ have the same distribution. 

\begin{theorem}\label{theorem2}
Let $X_A$ be the dataset for the task $T_A$ with the objective function $L_A$, and $X_B$ be the dataset for the task $T_B$ with the objective function $L_B$. Assume $X_A$ and $X_B$ have the same distribution. Consider an $\varepsilon$-approximation network $N$ trained using both datasets $X_A^{(1)}$ and $X_B^{(1)}$ respectively with the objective functions $L_A$ and $L_B$ to result weights $\theta_{A_{t}}$ and $\theta_{B_{t}}$ at time $t$. Under the same assumptions on the moment of gradient noise in SGD algorithm and the loss function stated in Theorem~\ref{theorem1}, the FTD from the task $A$ to the task $B$ computed from the Fisher Information matrices of the average of estimated weights up to the current time $t$ converges to a constant as $t \rightarrow \infty$. That is,
\begin{align}
    d_t = \frac{1}{\sqrt{2}}\norm{\Bar{F_A}_t^{1/2} - \Bar{F_B}_t^{1/2}}_F \xrightarrow[]{\mathcal{D}} \frac{1}{\sqrt{2}}\norm{{F_A^*}^{1/2} - {F_B^*}^{1/2}}_F,
\end{align}
where $\Bar{F}_{A_{t}}$ is given by $\Bar{F}_{A_{t}} = F(\bar{\theta}_{A_{t}})$ with $\bar{\theta}_{A_{t}} = \frac{1}{t}\sum_t \theta_{A_{t}}$, and $\Bar{F}_{B_{t}}$ is defined in a similar way.
\end{theorem}

\section{Neural Architecture Search}\label{sec:nas}
In this section, we apply Fisher task distance to the task-aware neural architecture search  (TA-NAS) framework~\cite{le2021task}, which finds the suitable architecture for a target task, based on a set of baseline tasks. Consider a set $A$ consisting of $K$ baseline tasks $T_i$ and its corresponding data set $X_i,$ denoted jointly by pairs $(T_i, X_i)$ for $i=1,2,\ldots,K$. Below, TA-NAS framework with the FTD is presented for finding a well-performing architecture for a target task $b,$ denoted by the pair $(T_b, X_b),$ based on the knowledge of architectures of these $K$ learned baseline tasks. We assume that $X_1, X_2, ..., X_K$ and $X_b$ are known, and their data points are of the same dimension. The pipeline of the TA-NAS, whose pseudo-code is given by Algorithm~\ref{alg2}, is summarized below:
\begin{enumerate}
    \item \textbf{Fisher Task Distance.} First, the FTD of each learned baseline task $a \in A$ to the target task $b$ is computed. The closest baseline task $a^*$, based on the computed distances, is returned.
    \item \textbf{Search Space.} Next, a suitable search space for the target task $b$ is determined based on the closest task architecture.
    \item \textbf{Search Algorithm.} Finally, the FUSE  algorithm performs a search within this space to find a well-performing architecture for the target task $b$.
\end{enumerate}

\begin{algorithm}[t]
\SetKwInput{KwInput}{Input}         
\SetKwInput{KwOutput}{Output}       
\SetKwFunction{FUSE}{FUSE}
\SetKwFunction{FMain}{Main}
\DontPrintSemicolon

\KwData{$A = \{(T_{1},X_{1}),\ldots,(T_{K},X_{K})\}$, $b = (T_b, X_b)$}
\KwInput{$\varepsilon$-approx. network $N$, \# of candidates $C$, $\alpha = 1/|C|$, baseline spaces $\{S_1,\ldots, S_K\}$}
\KwOutput{Best architecture for $b$}

    \SetKwProg{Fn}{Function}{:}{}
    \Fn{\FUSE{candidates $C$, data $X$}}{
        Relax the output of C (using Softmax function): $\displaystyle \Bar{c}(X) = \underset{c\in C}{\sum} \frac{\exp{(\alpha_c)}}{\underset{c'\in C}{\sum} \exp{(\alpha_{c'})}} c(X)$\;
        \While{$\alpha$ has not converged}{
            Update $C$ by descending $\nabla_{w} \mathcal{L}_{train}(w; \alpha, \Bar{c})$\;
            Update $\alpha$ by descending $\nabla_{\alpha} \mathcal{L}_{val}(\alpha; w, \Bar{c})$\;
        }
        \KwRet $c^* = \underset{c \in C}{\mathrm{argmin}}\ \alpha_c$\;
    }
    
    \SetKwProg{Fn}{Function}{:}{\KwRet}
    \Fn{\FMain}{
        \For{$a \in A$}{
            Compute Fisher Task Distance from $a$ to $b$:\;
            $d[a,b] = \TaskDistance(X_a^{(1)},X_b, N)$
        }
        Select closest task: $a^* = \underset{a \in A}{\mathrm{argmin}}\ d[a,b]$\;
    
        Define search space $ S = S_{a^*} $\;
        \While{criteria not met}{
            Sample $C$ candidates $\in$ S\;
            $c^* = \FUSE \big( (C \cup c^*), X_b \big)$\;
        }
        \KwRet best architecture $c^*$\;
    }
\caption{TA-NAS framework}
\label{alg2}
\end{algorithm}

\subsection{Search Space}
\label{ssec:search-space}
In an analogous manner to recent NAS techniques~\cite{liu2018darts,dong2020bench}, our architecture search space is defined by cells and skeletons. A cell, illustrated in Figure~\ref{fig-cell}, is a densely connected directed-acyclic graph (DAG) of nodes, where nodes are connected by operations. Each node has $2$ inputs and $1$ output. The operations (e.g., identity, zero, convolution, pooling) are normally set so that the dimension of the output is the same as that of the input. Additionally, a skeleton, illustrated in Figure~\ref{fig-skeleton}, is a structure consisting of multiple cells  and other operations stacked together, forming a complete architecture. In our framework, the search space for the task is defined in terms of the cells and operations.

Given the computed task similarity obtained by FTD, we can identify the closest baseline task(s) to the target task. Build upon this knowledge, we construct the search space for the target task based on the search space of the closest baseline task(s). Since the architecture search space of TA-NAS for a target task is restricted only to the space of the most related task from the baseline tasks, the search algorithm performs efficiently and requires fewer computational resources as it is demonstrated in our experiments on classification tasks in MNIST, CIFAR-10, CIFAR-100, ImageNet datasets. 

\begin{figure}[t]
    \centering
    \begin{subfigure}{0.45\textwidth}
        \centering
        \includegraphics[width=7cm]{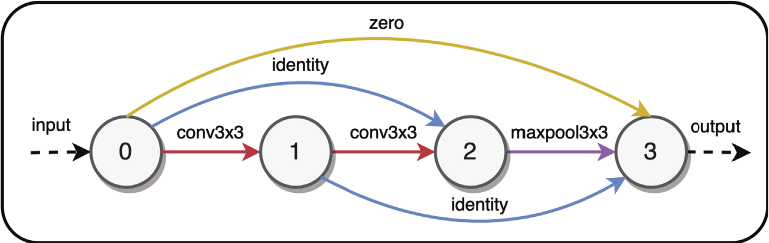}
        \captionof{figure}{A cell structure}
        \label{fig-cell}
        \vspace{0.2cm}
    \end{subfigure}
    \begin{subfigure}{0.45\textwidth}
        \centering
        \includegraphics[width=7cm]{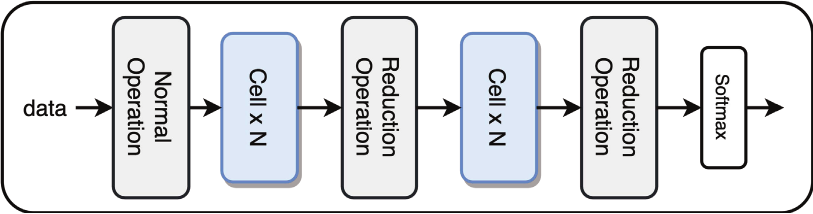}
        \captionof{figure}{A skeleton structure}
        \label{fig-skeleton}
    \end{subfigure}
    \caption{An example of the cell and the skeleton.}
\end{figure}

\subsection{Search Algorithm}
\label{ssec:IDARTS}
The Fusion Search (FUSE) is a architecture search algorithm that considers the network candidates as a whole with weighted-sum output, and performs the optimization using gradient descent. Let $C$ be the set of candidate networks which are drawn from the search space. Given $c \in C$ and training data $X$, denote by $c(X)$ the output of the network candidate $c$. The FUSE algorithm, is based on the continuous relaxation of the weighted outputs of the network candidates. It is capable of searching through all candidates in the relaxation space, and identifying the most promising network without fully training them. We defined the relaxed output $\Bar{c}$ for all of the candidates in $C$ by convex combination, where each weight in the combination given by exponential weights:
\begin{equation}
    \Bar{c}(X) = \sum_{c\in C}\frac{\exp{(\alpha_c)}}{\sum_{c'\in C}\exp{(\alpha_{c'})}} c(X),
\end{equation}
where $\Bar{c}$ is the weighted output of network candidates in $C$, and $\alpha_c$ is a continuous variable that assigned to candidate $c$'s output. Next, we conduct our search for the most promising candidate in $C$ by jointly training the network candidates, where the output is given by $\Bar{c}$, and optimizing their $\alpha$ coefficients. Let $X_{train}$, $X_{val}$ be the training and validation data set. The training procedure is based on alternative minimization and can be divided into: (i) freeze $\alpha$ coefficients, jointly train network candidates, (ii) freeze network candidates, update $\alpha$ coefficients. Initially, $\alpha$ coefficients are set to $1/|C|$ for all candidates. While freezing $\alpha$, we update the weights in network candidates by jointly train the relaxed output $\Bar{c}$ with cross-validation loss on training data:
\begin{equation}{\label{eq3}}
    \min_w \mathcal{L}_{train}(w; \alpha, \Bar{c}, X_{train}),
\end{equation}

where $w$ are weights of network candidates in $C$. Next, the weights in those candidates are fixed while we update the $\alpha$ coefficients on validation data:
\begin{equation}{\label{eq4}}
    \min_\alpha \mathcal{L}_{val}(\alpha; w, \Bar{c}, X_{val}).
\end{equation}
These steps are repeated until $\alpha$ converges or certain criteria (e.g., a number of iterations, $\alpha$ is greater than a predetermined threshold) are met. The most promising candidate will be selected by: $c^* = \arg\max_{c \in C} \alpha_c$. This training procedure will select the best candidate network among candidates in $C$ without fully training all of them. In order to go through the entire search space, this process is repeated, as more candidates can be drawn from the search space. The search stops when certain criteria, such as the number of iterations, the performance of the current most promising candidate, is met.

\section{Experimental Study}\label{sec:experiment}
In this section, we conduct experiments to show the consistency of our Fisher task distance (FTD), as well as its applications in neural architecture search (NAS) and transfer learning (TL). 


\subsection{Detail of Experiments}
In our experiments, the first step is to represent each task and its corresponding dataset by an $\varepsilon$-approximation network. To this end, we train the $\varepsilon$-approximation network with the balanced data from each task. For classification tasks (e.g., tasks in  MNIST~\cite{lecun2010mnist}, CIFAR-10~\cite{krizhevsky2009learning}, CIFAR-100~\cite{krizhevsky2009learning}, ImageNet~\cite{ILSVRC15}), three different architectures (e.g., VGG-16~\cite{simonyan2014very}, Resnet-18~\cite{he2016deep}, DenseNet-121~\cite{huang2017densely}) are chosen as $\varepsilon$-approximation network architectures. The training procedure is conducted in $100$ epochs, with Adam optimizer~\cite{kingma2014adam}, a batch size is set to $128$, and cross-validation loss. For image processing tasks in Taskonomy dataset~\cite{zamir2018taskonomy}, we use an autoencoder as the $\varepsilon$-approximation network. The encoder of the autoencoder consists of one convolutional layer, and two linear layers. The convolutional layer has $3$ input channels and $16$ output channels with the kernel size equals to $5$. We also use the zero padding of size $2$, stride of size $4$, and dilation equals to $1$. The first linear layer has the size of $(262144, 512)$, and the second linear layer has the size of $(512, 128)$. The training procedure is conducted in $20$ epochs with Adam optimizer~\cite{kingma2014adam}, a batch size of $64$, and mean-square error loss.

In order to construct the dictionary for baseline tasks, we need to perform the architecture search for these tasks using general search space. This space is defined by cell structures, consisting of $3$ or $4$ nodes, and $10$ operations (i.e., zero, identity, maxpool3x3, avepool3x3, conv3x3, conv5x5, conv7x7, dil-conv3x3, dil-conv5x5, conv7x1-1x7). After the best cell for each baseline task is founded, we save the structures and operations to the dictionary. Next, we apply the task-aware neural architecture search (TA-NAS) framework~\cite{le2021task} to find the best architecture for a target task, given a knowledge of learned baseline tasks.

First, consider a set $A$ consisting of $K$ baseline tasks $T_i$ and its corresponding data set $X_i,$ denoted jointly by pairs $(T_i, X_i)$ for $i=1,2,\ldots,K$ where $X_i=X_i^{(1)}\cup X_i^{(2)}$ with $X_i^{(1)}$ and $X_i^{(2)}$ denote the training and the test data, respectively. Now, suppose that for a given $\varepsilon$, the $\varepsilon$-approximation networks for the tasks, $(T_i, X_i)$ for $i=1,2, \cdots, K$ denoted by $N_1, N_2, ..., N_k$, respectively, where $\varepsilon$ is selected such that $\min_{i\in\{1,2,\dots,K\}} \mathcal{P}_{N_i}(T_i,X_i^{(2)})\geq 1-\varepsilon$. Here, the Fisher task distance from each of the baseline tasks to the target is computed. The baseline task with the smallest distance (i.e., the most related task) will be selected and used to construct the restricted search space for the target task. Lastly, the FUSE algorithm is applied to search for the best architecture for the target task from the restricted search space.

The experiment is conducted using NVIDIA GeForce RTX 3090. The source code for the experiments is available at: {\color{blue} \text{https://github.com/lephuoccat/Fisher-Information-NAS}}. 
Next, we provide the detail of the TA-NAS framework.

\begin{figure*}
    \centering
    \begin{subfigure}{0.33\textwidth}
        \centering
        \includegraphics[height=4.5cm]{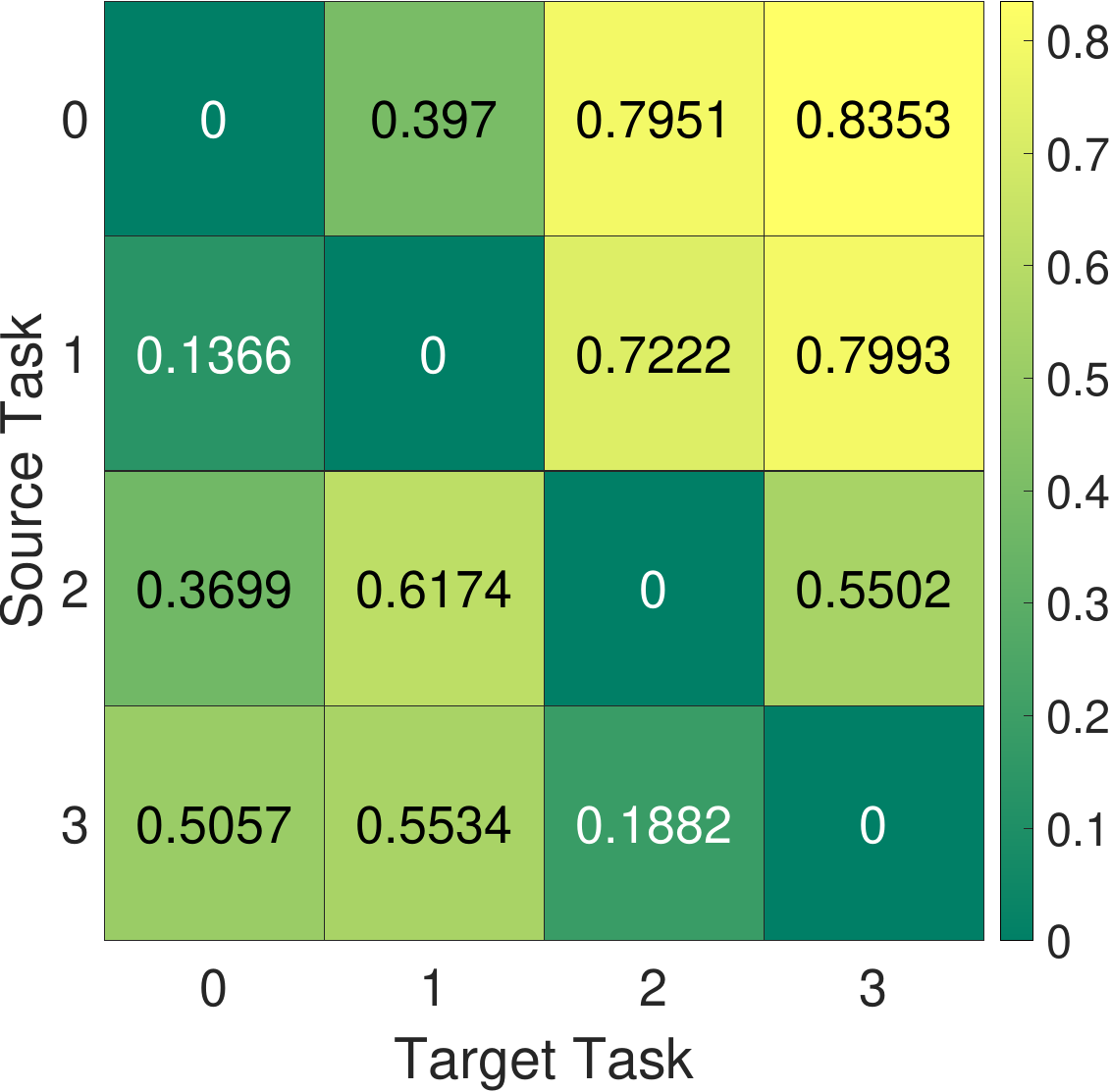}
    \end{subfigure}
    \begin{subfigure}{0.33\textwidth}
        \centering
        \includegraphics[height=4.5cm]{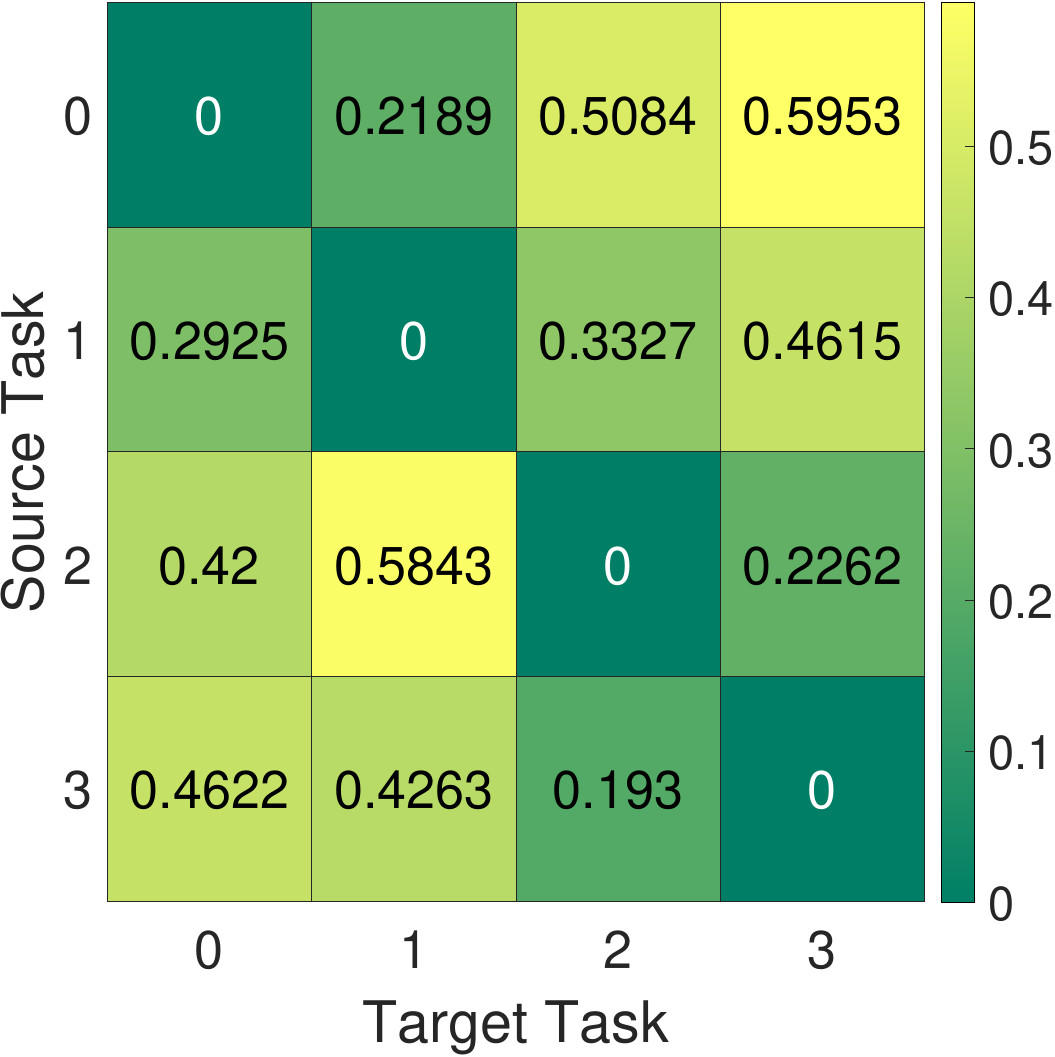}
    \end{subfigure}
    \begin{subfigure}{0.33\textwidth}
        \centering
        \includegraphics[height=4.5cm]{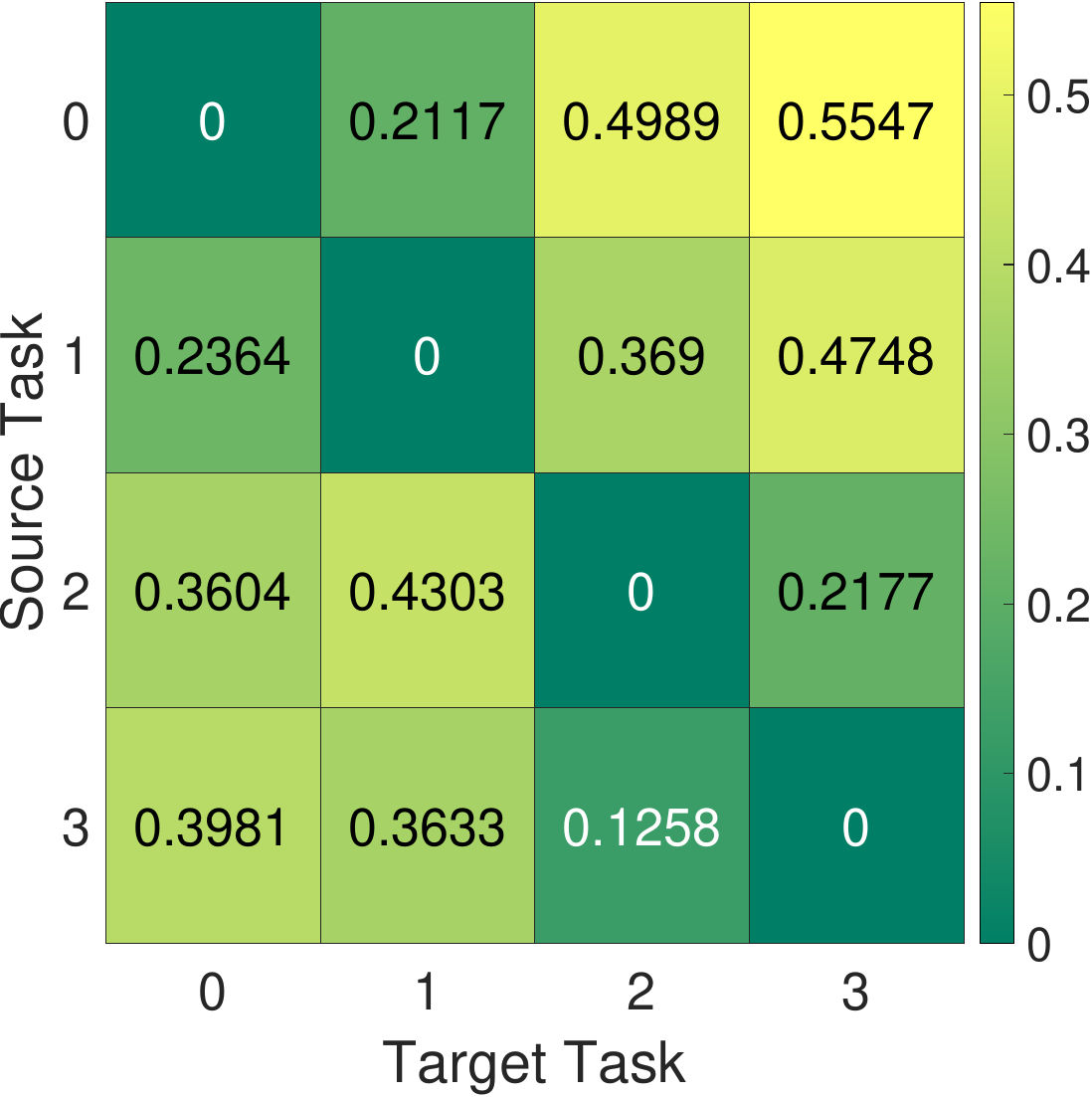}
    \end{subfigure}
    
    \begin{subfigure}{0.33\textwidth}
        \centering
        \includegraphics[height=4.5cm]{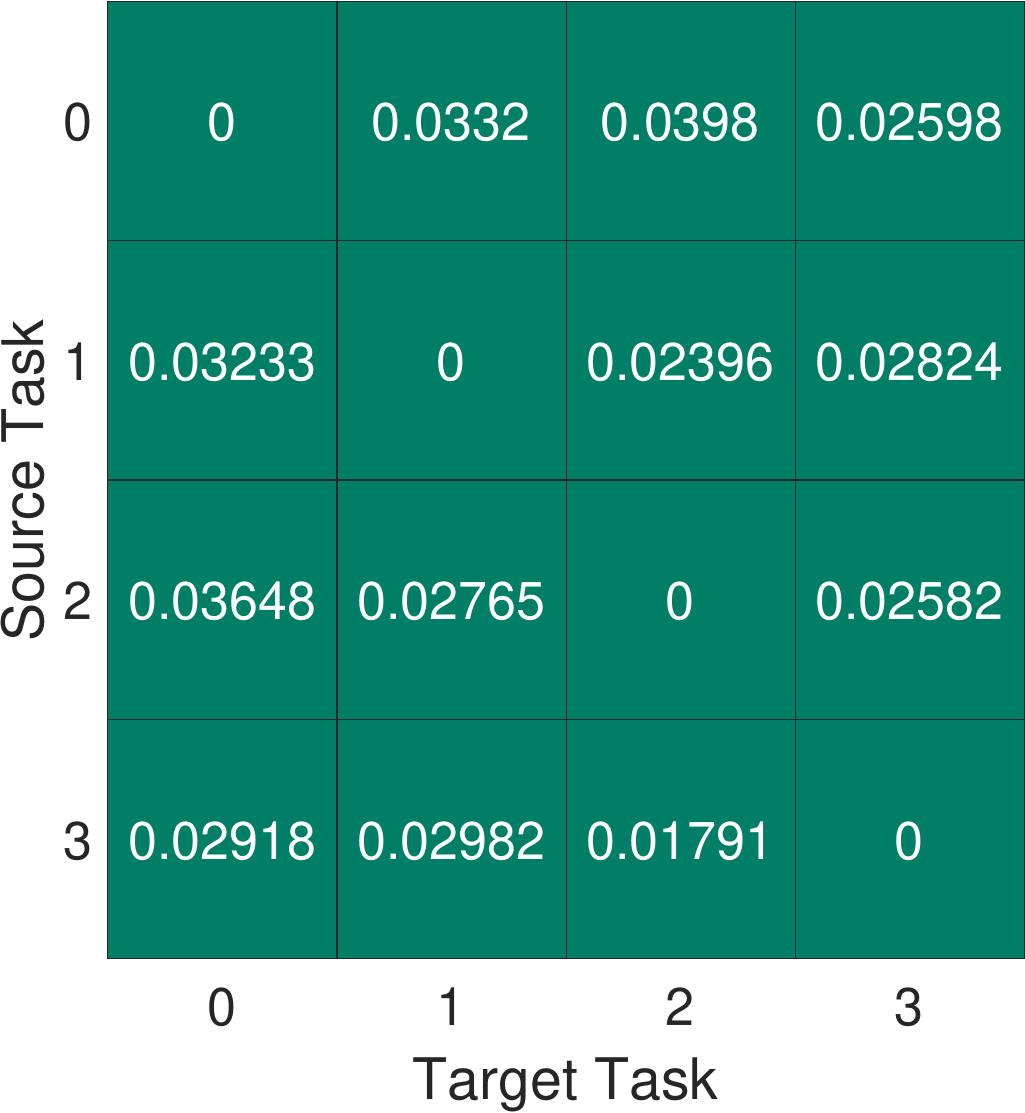}
        \caption{VGG-16}
    \end{subfigure}
    \begin{subfigure}{0.33\textwidth}
        \centering
        \includegraphics[height=4.5cm]{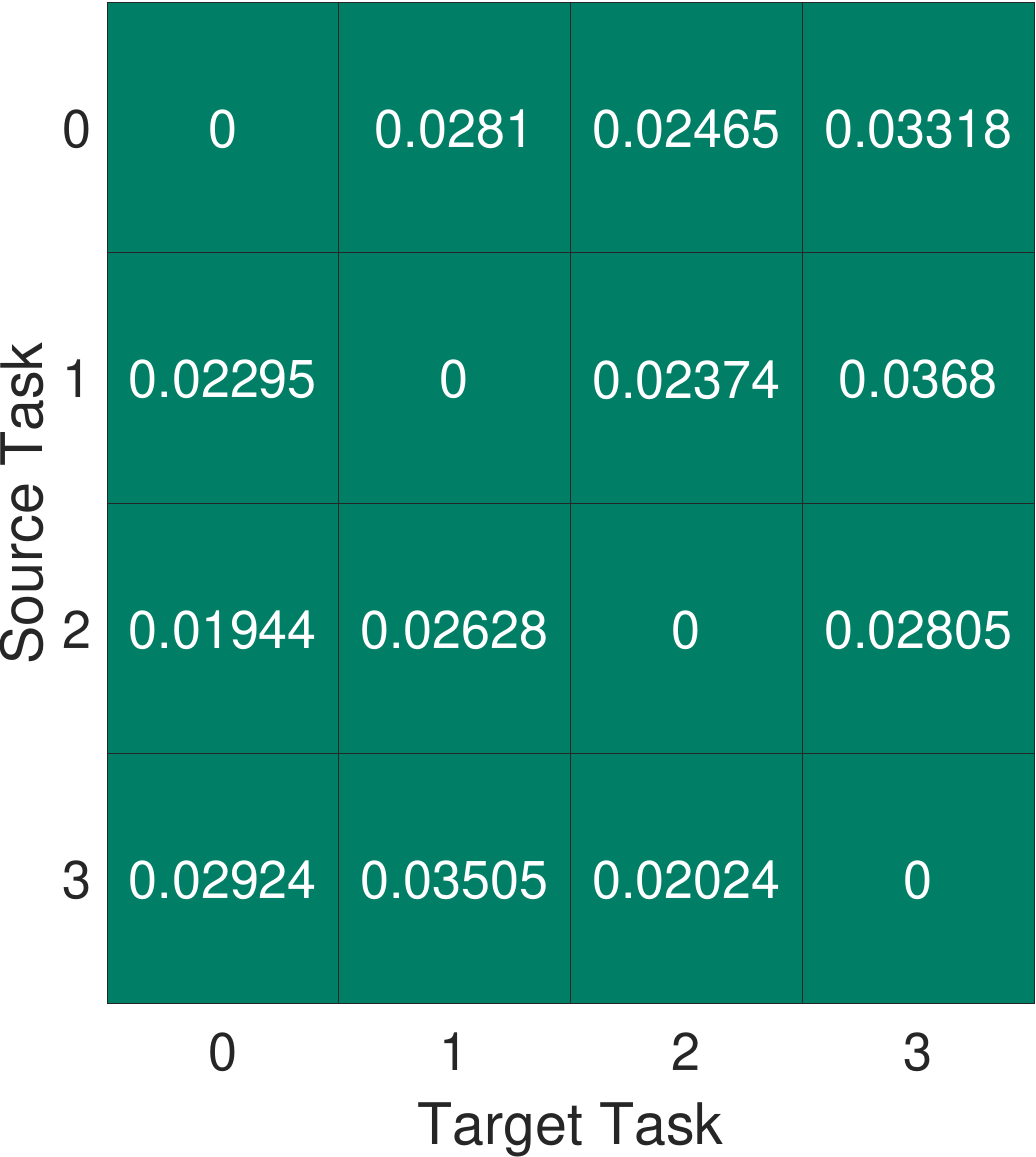}
        \caption{ResNet-18}
    \end{subfigure}
    \begin{subfigure}{0.33\textwidth}
        \centering
        \includegraphics[height=4.5cm]{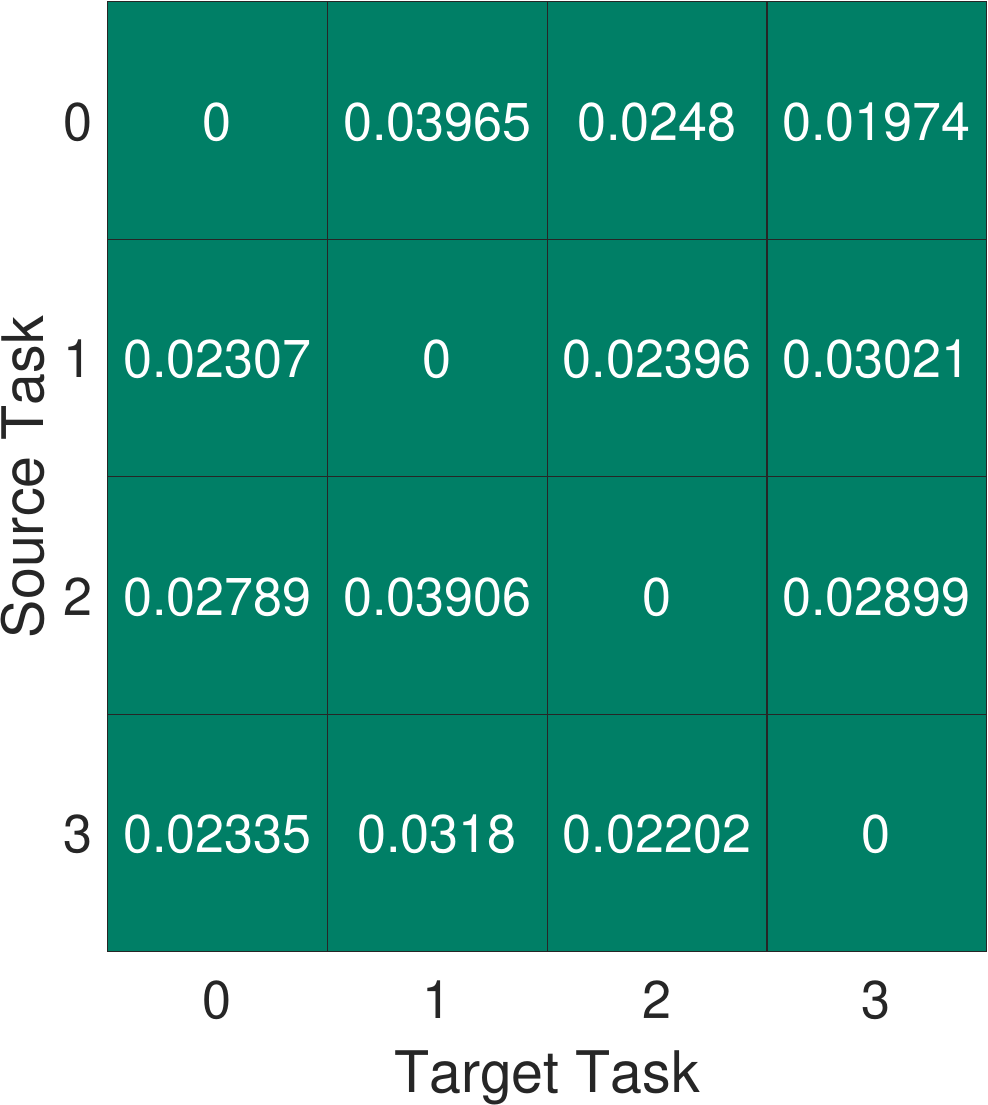}
        \caption{DenseNet-121}
    \end{subfigure}
    \caption{Distance from source tasks to the target tasks on MNIST. The top row shows the mean values and the bottom row denotes the standard deviation of distances between classification tasks over 10 different trials.}
    \label{fig-mnist}
\end{figure*}

\begin{figure*}
    \centering
    \begin{subfigure}{.33\textwidth}
        \centering
        \includegraphics[height=4.5cm]{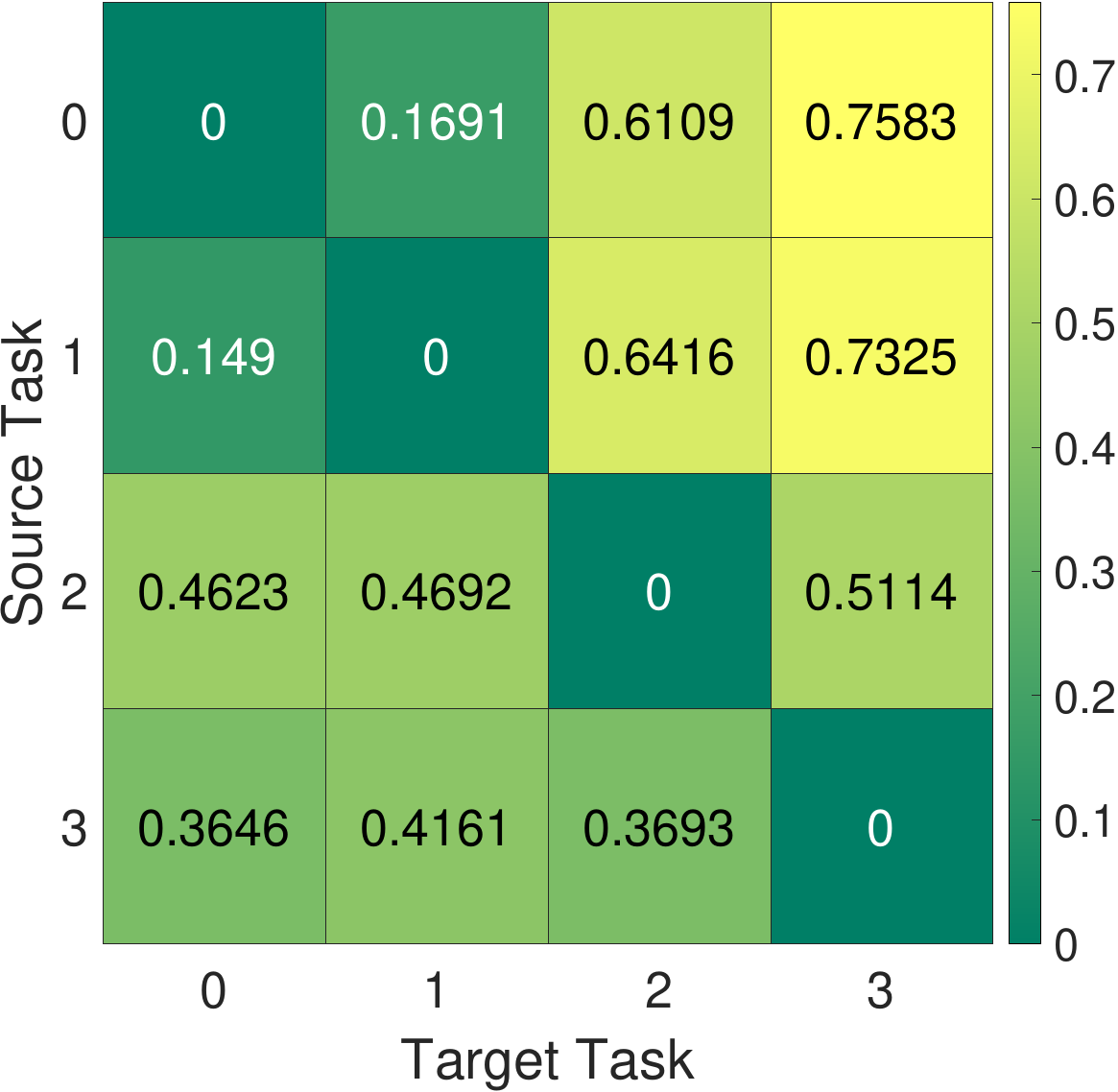}
    \end{subfigure}
    \begin{subfigure}{.33\textwidth}
        \centering
        \includegraphics[height=4.5cm]{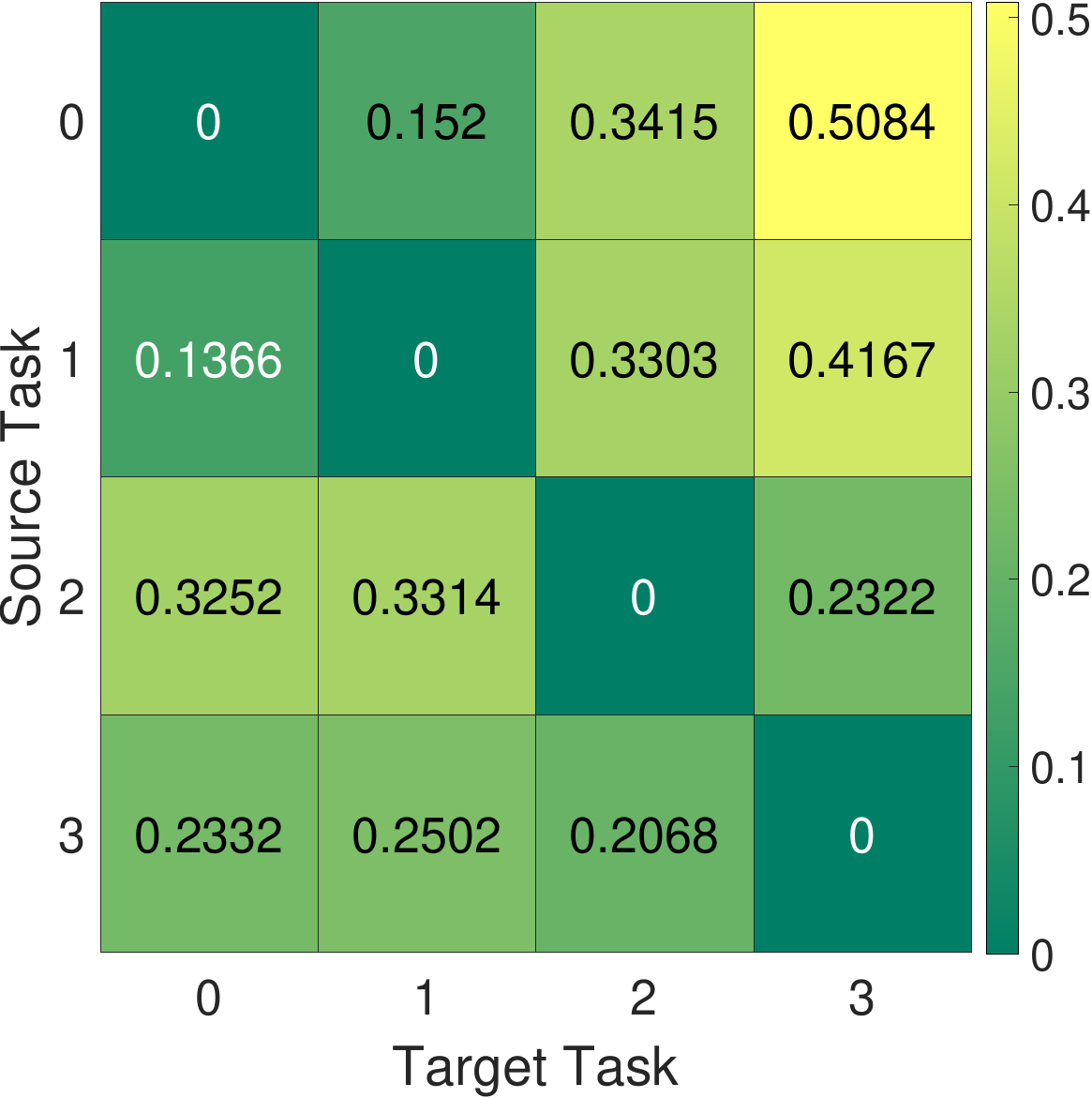}
    \end{subfigure}
    \begin{subfigure}{.33\textwidth}
        \centering
        \includegraphics[height=4.5cm]{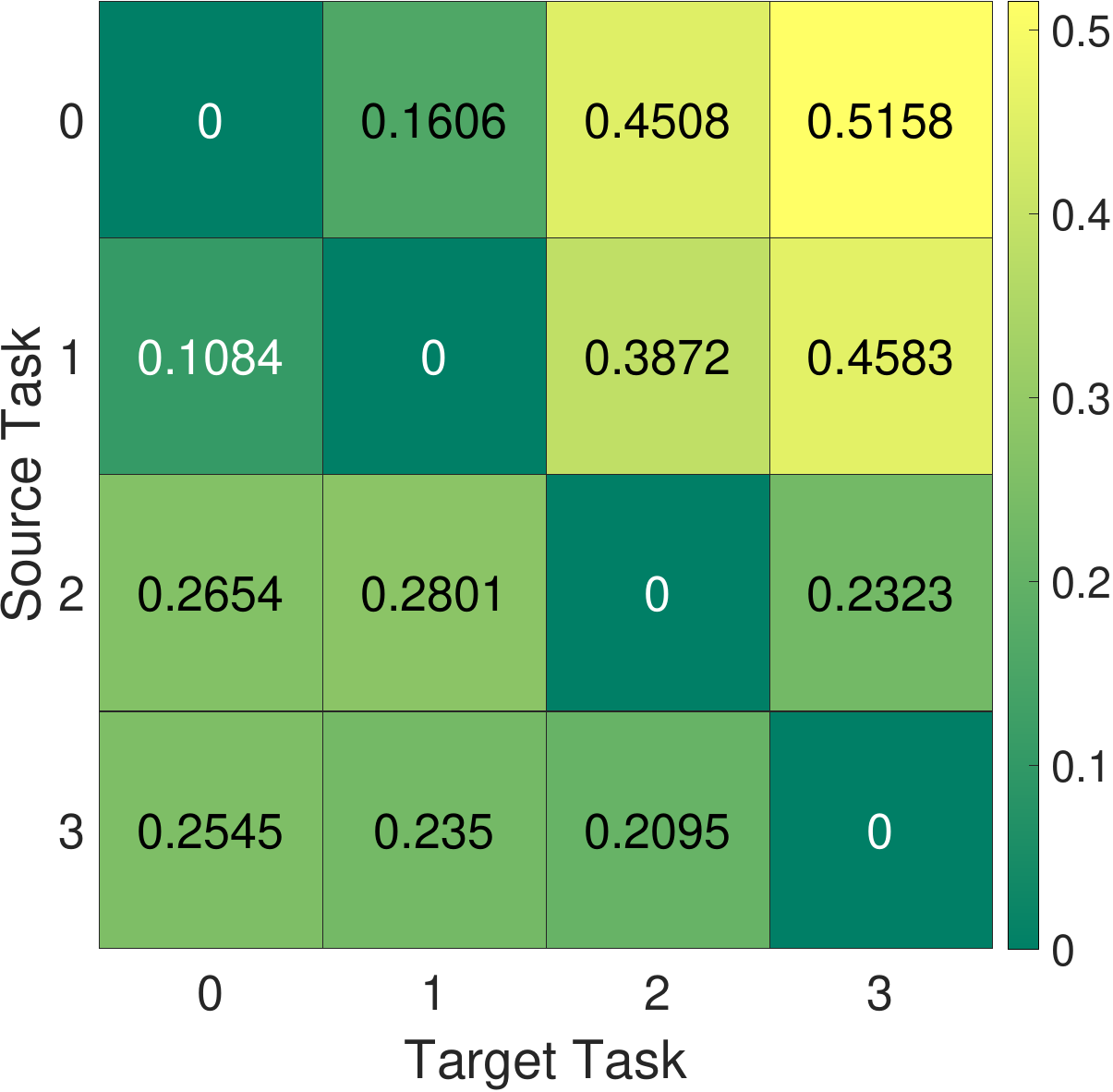}
    \end{subfigure}
    
    \begin{subfigure}{.33\textwidth}
        \centering
        \includegraphics[height=4.5cm]{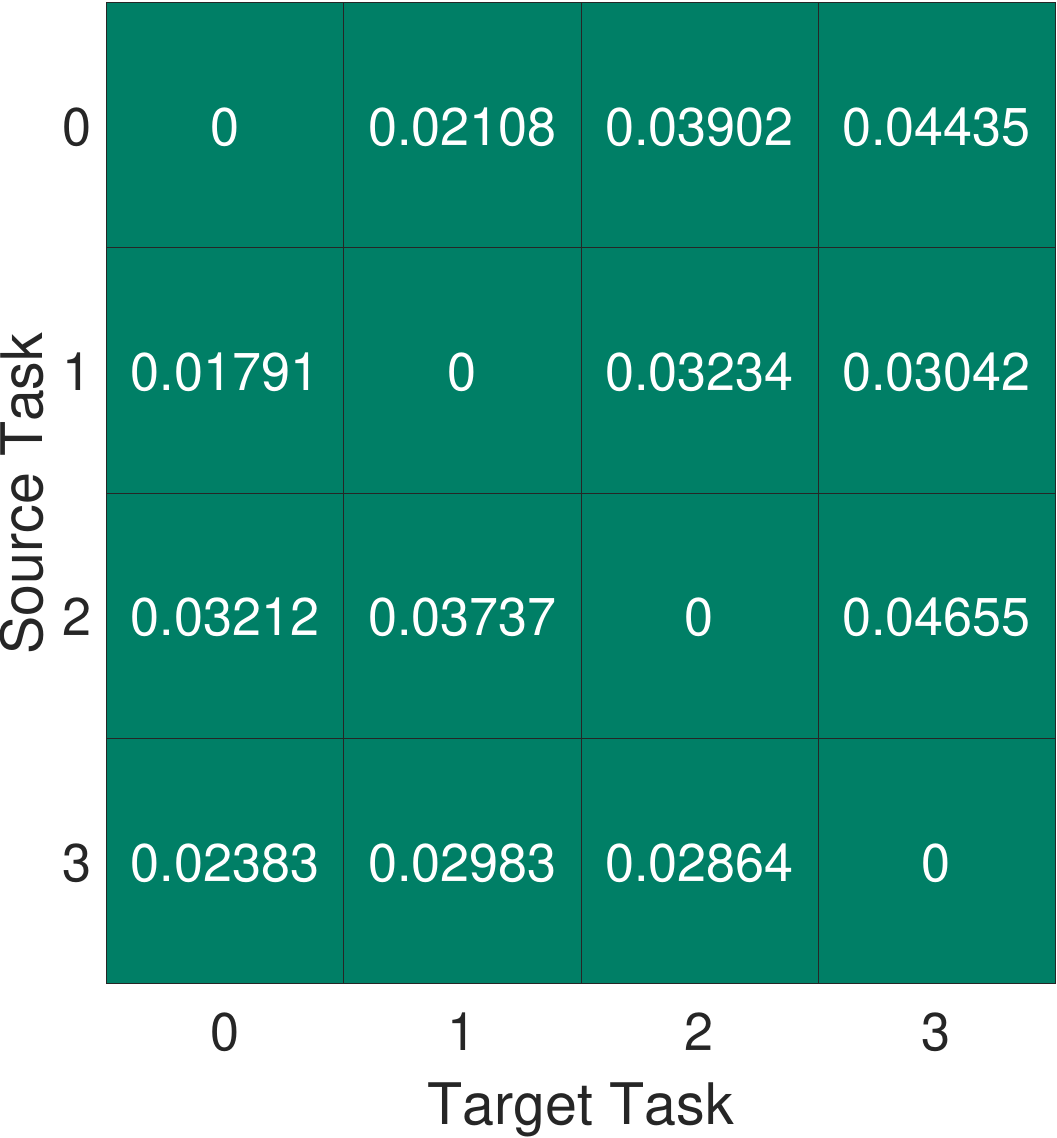}
        \caption{VGG-16}
    \end{subfigure}
    \begin{subfigure}{0.33\textwidth}
        \centering
        \includegraphics[height=4.5cm]{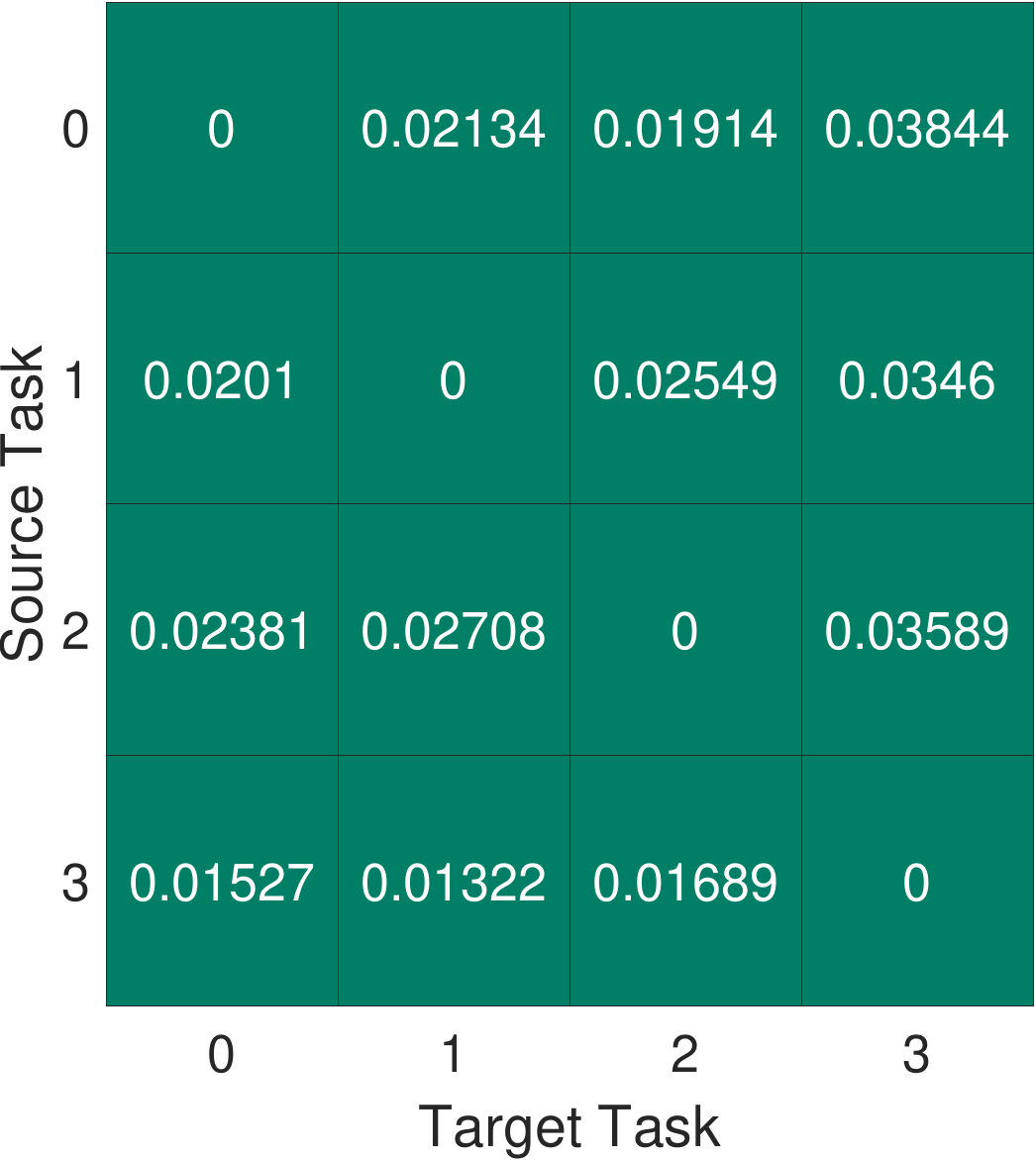}
        \caption{ResNet-18}
    \end{subfigure}
    \begin{subfigure}{.33\textwidth}
        \centering
        \includegraphics[height=4.5cm]{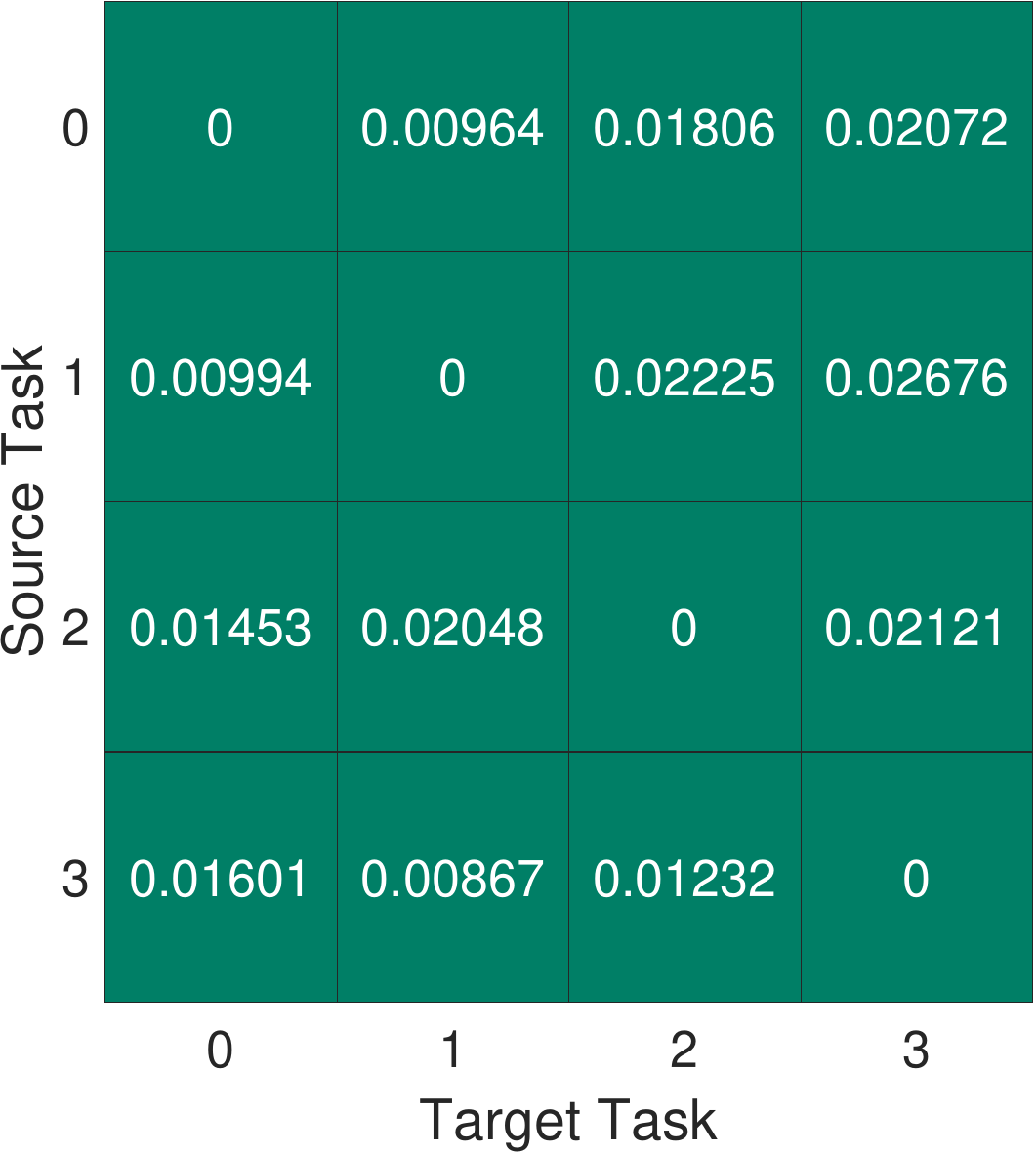}
        \caption{DenseNet-121}
    \end{subfigure}
    \caption{Distance from source tasks to the target tasks on CIFAR-10. The top row shows the mean values and the bottom row denotes the standard deviation of distances between classification tasks over 10 different trials.}
    \label{fig-cifar10}
\end{figure*}

\begin{figure*}
    \centering
    \begin{subfigure}{.33\textwidth}
        \centering
        \includegraphics[height=4.5cm]{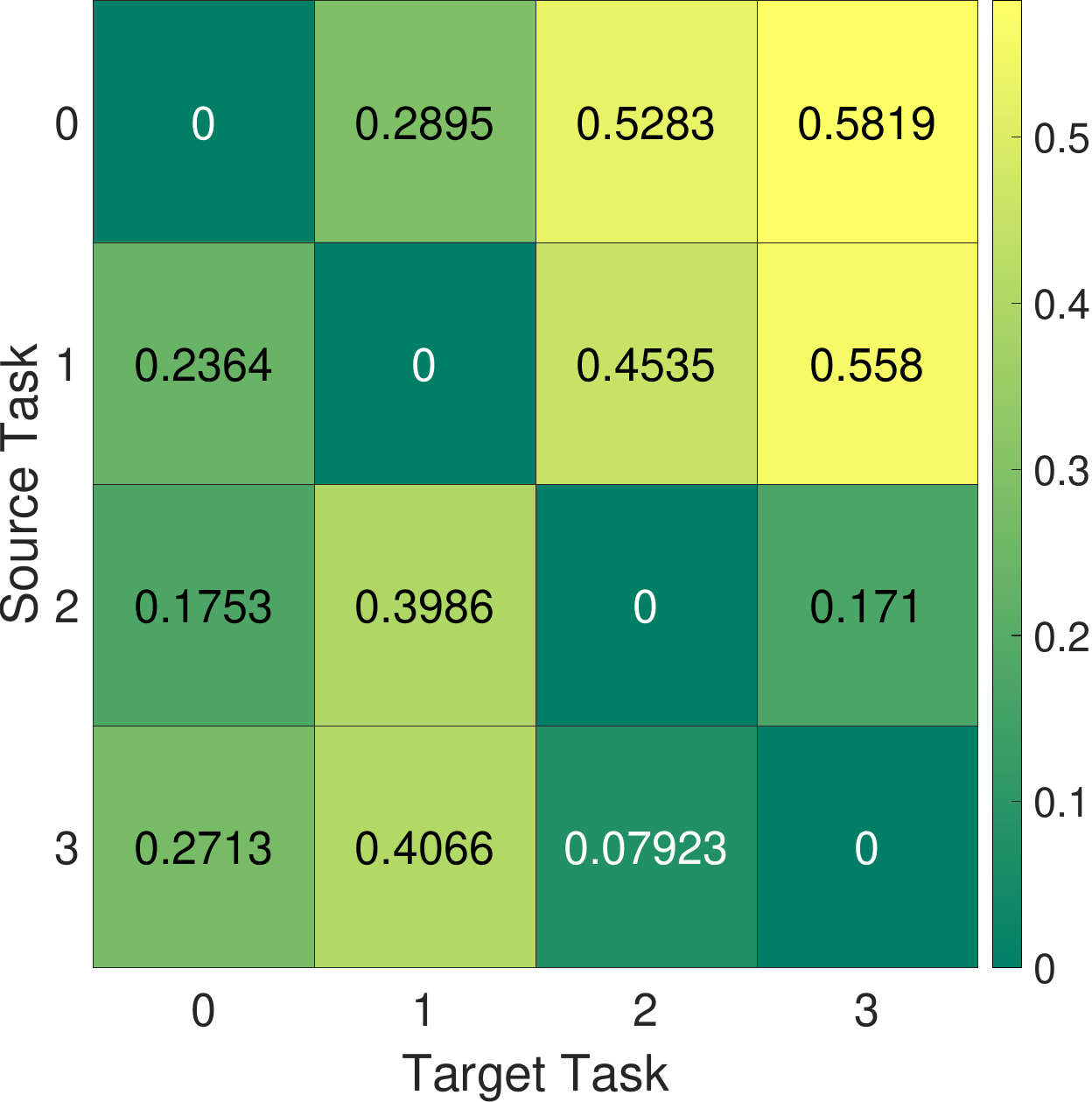}
    \end{subfigure}
    \begin{subfigure}{.33\textwidth}
        \centering
        \includegraphics[height=4.5cm]{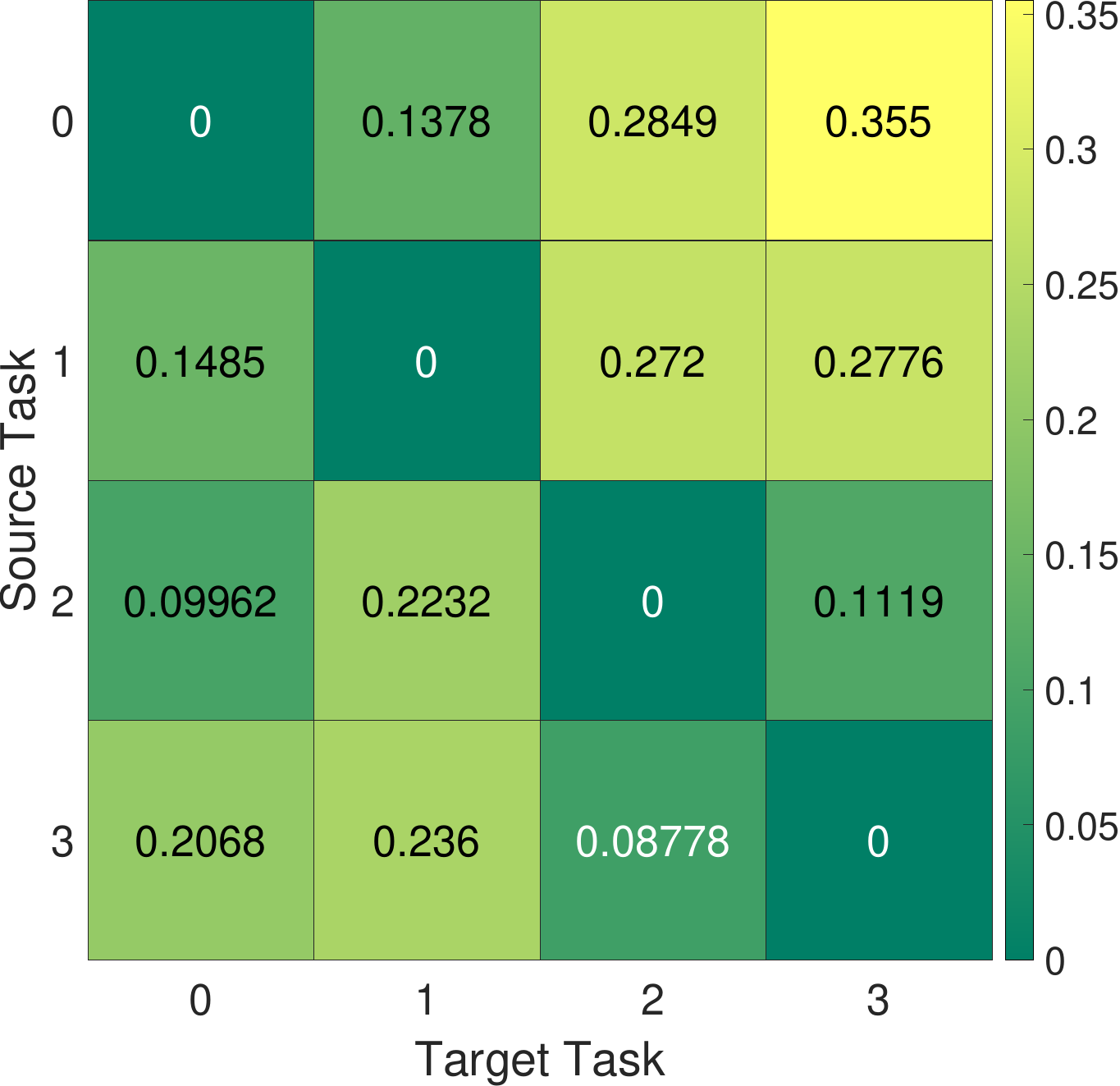}
    \end{subfigure}
    \begin{subfigure}{.33\textwidth}
        \centering
        \includegraphics[height=4.5cm]{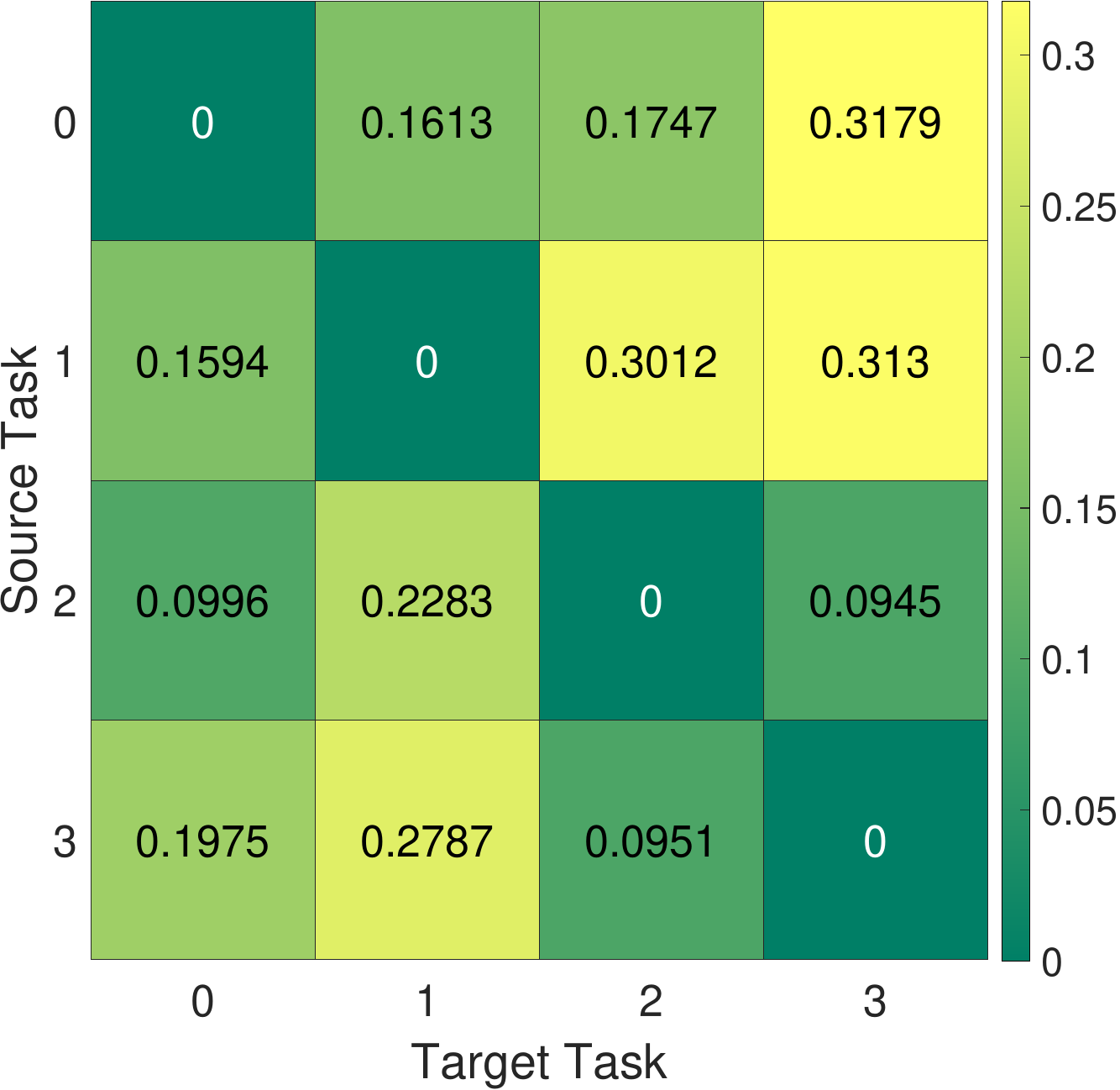}
    \end{subfigure}

    \begin{subfigure}{.33\textwidth}
        \centering
        \includegraphics[height=4.5cm]{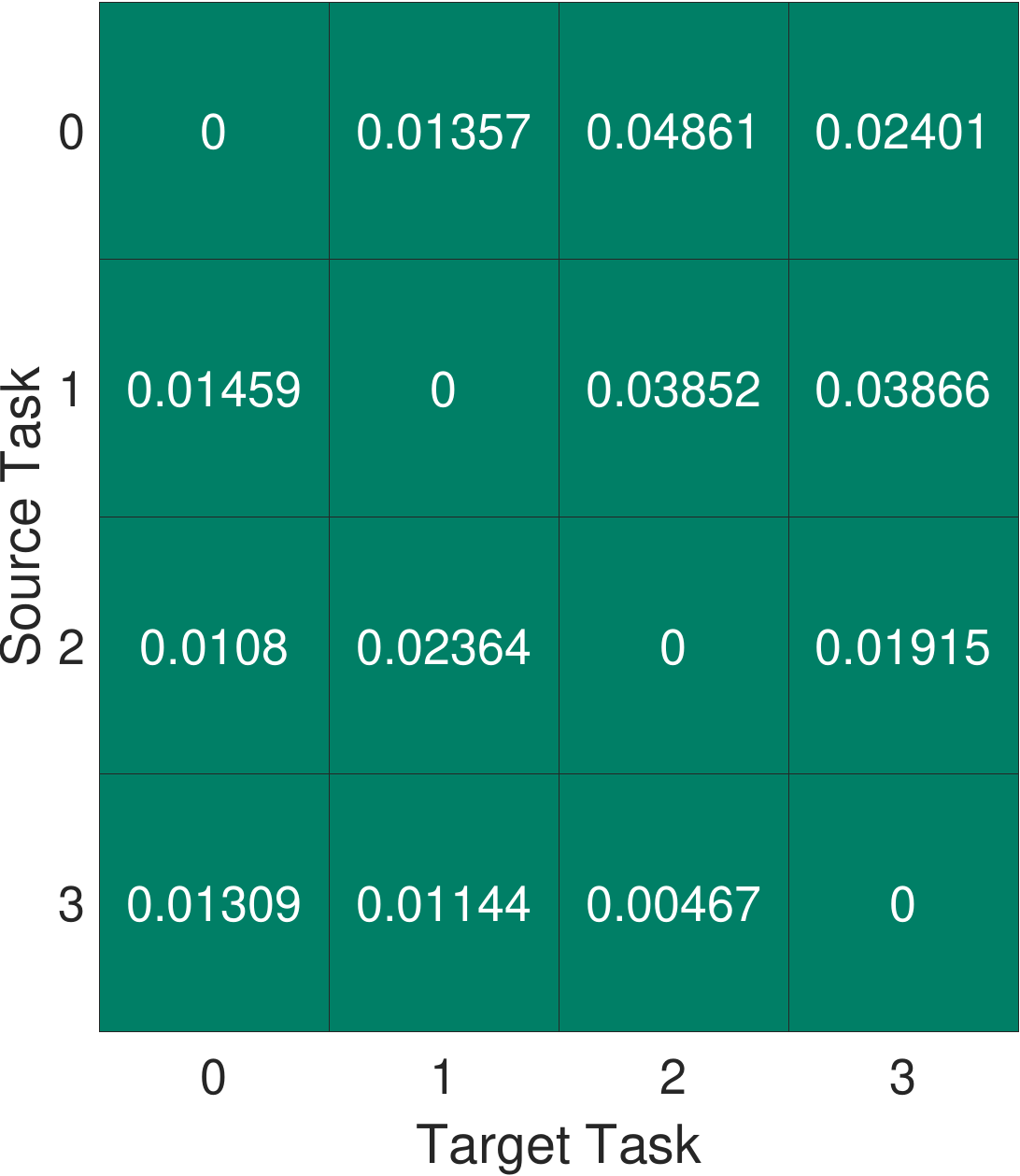}
        \caption{VGG-16}
    \end{subfigure}
    \begin{subfigure}{0.33\textwidth}
        \centering
        \includegraphics[height=4.5cm]{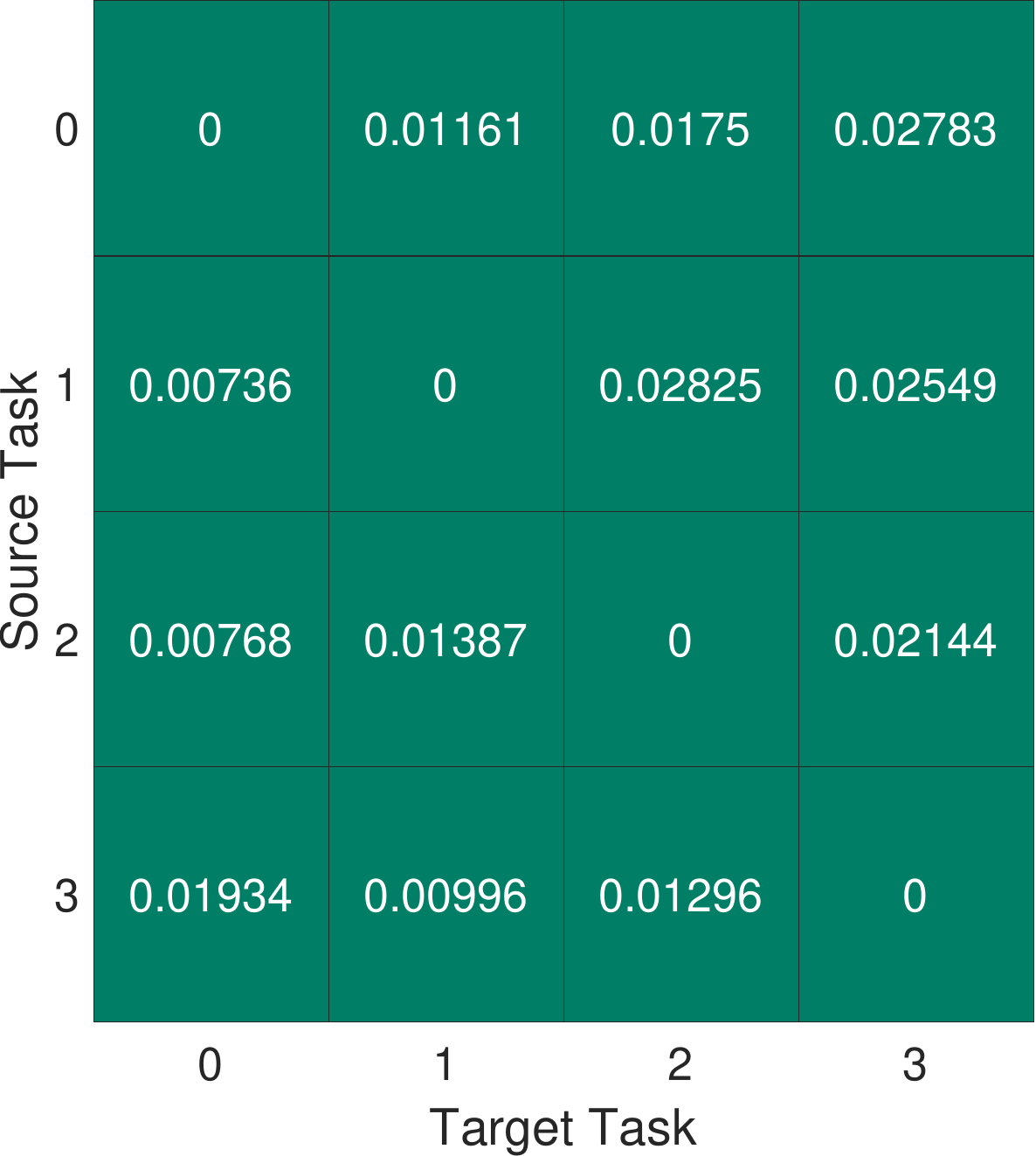}
        \caption{ResNet-18}
    \end{subfigure}
    \begin{subfigure}{.33\textwidth}
        \centering
        \includegraphics[height=4.5cm]{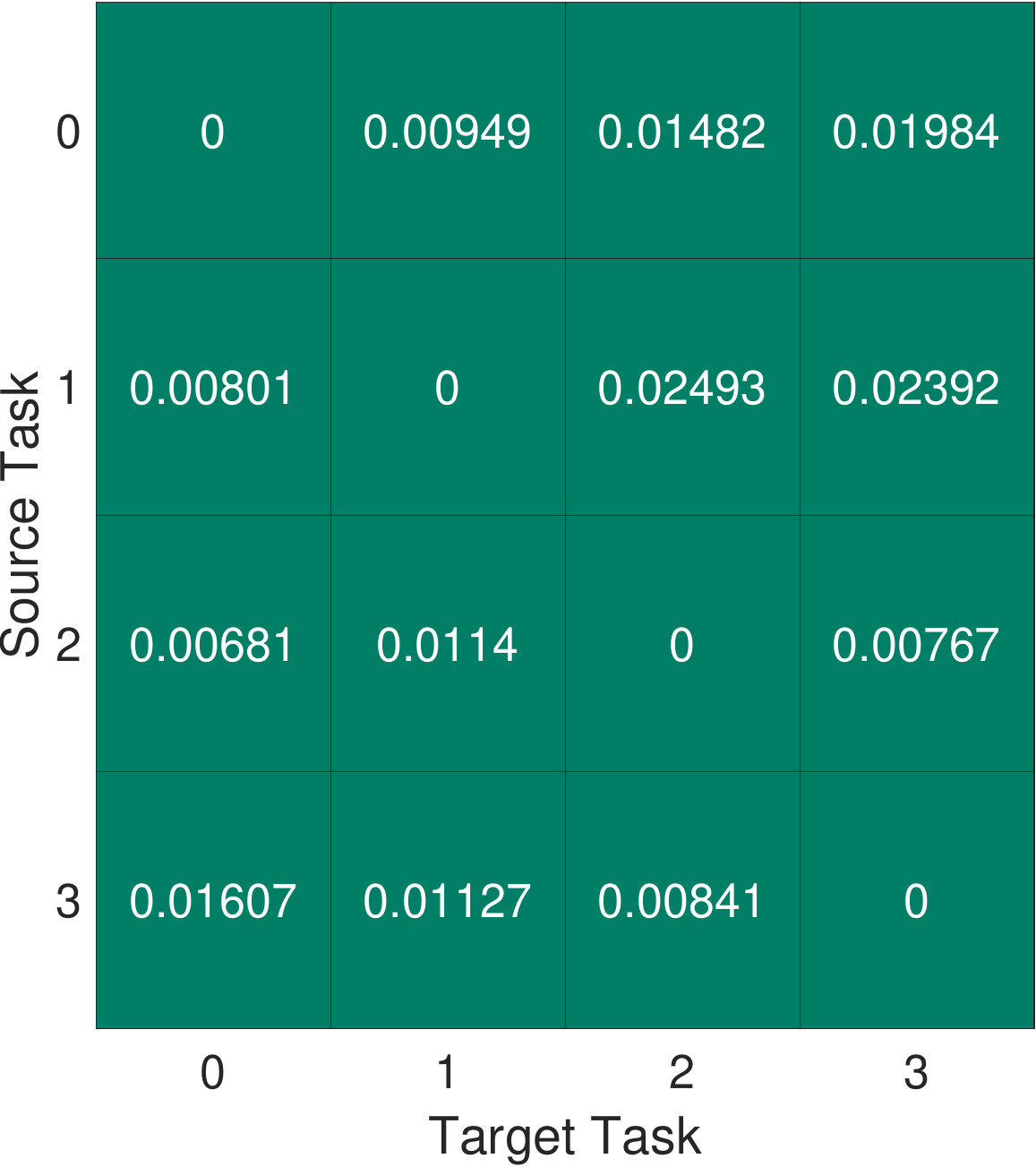}
        \caption{DenseNet-121}
    \end{subfigure}
    \caption{Distance from source tasks to the target tasks on CIFAR-100. The top row shows the mean values and the bottom row denotes the standard deviation of distances between classification tasks over 10 different trials.}
    \label{fig-cifar-100}
\end{figure*}

\begin{figure*}
    \centering
    \begin{subfigure}{.33\textwidth}
        \centering
        \includegraphics[height=4.5cm]{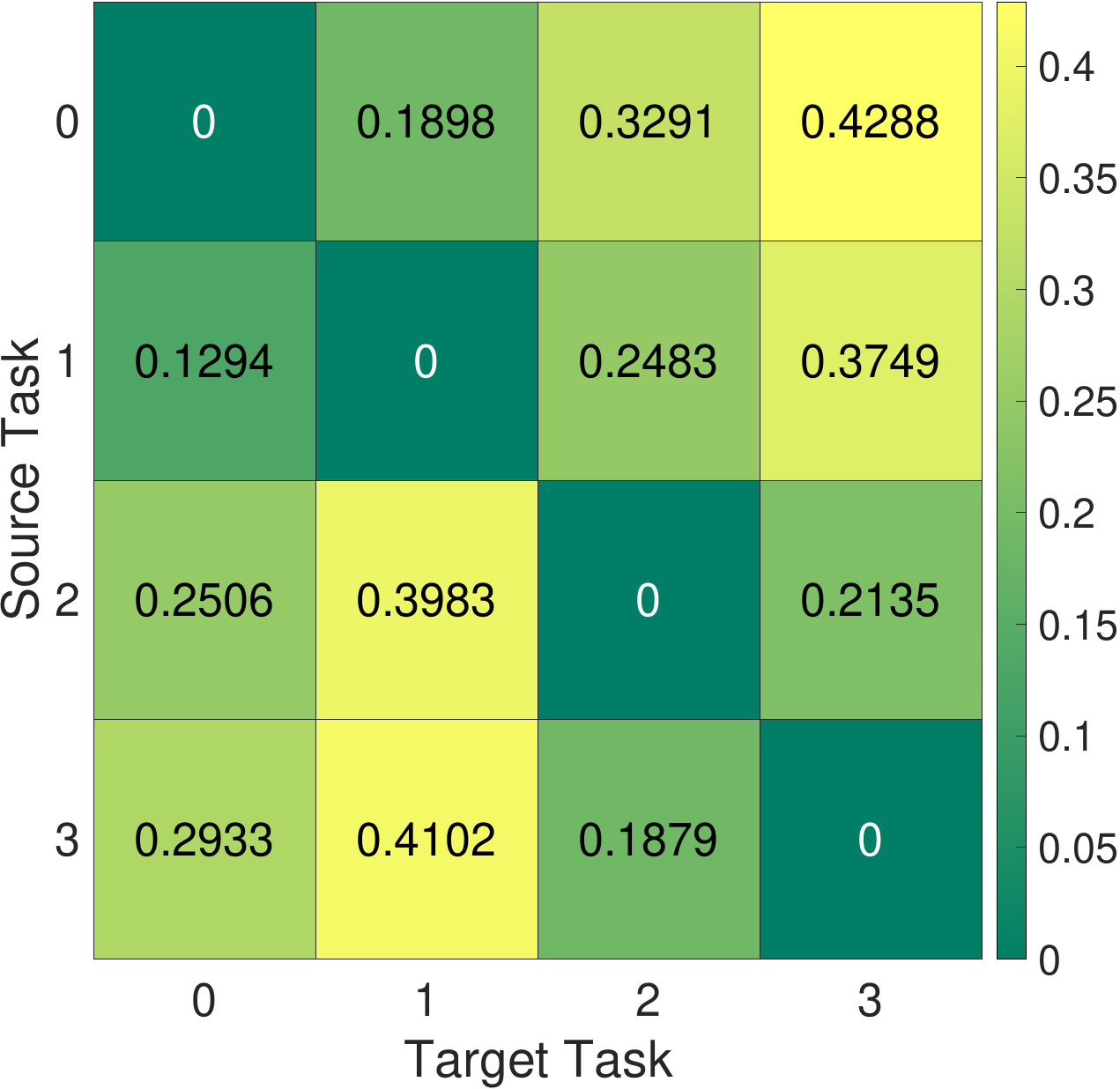}
    \end{subfigure}
    \begin{subfigure}{.33\textwidth}
        \centering
        \includegraphics[height=4.5cm]{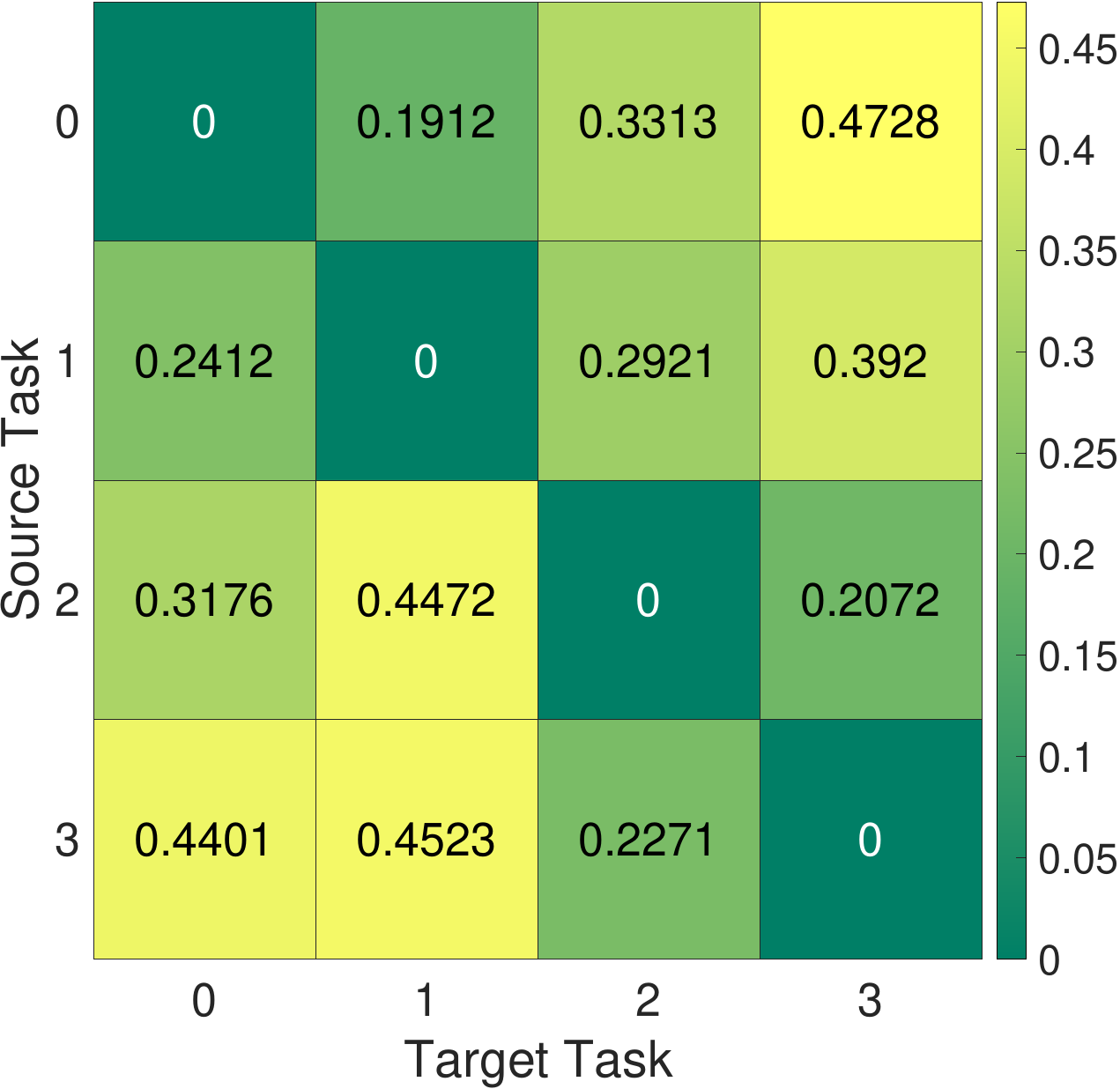}
    \end{subfigure}
    \begin{subfigure}{.33\textwidth}
        \centering
        \includegraphics[height=4.5cm]{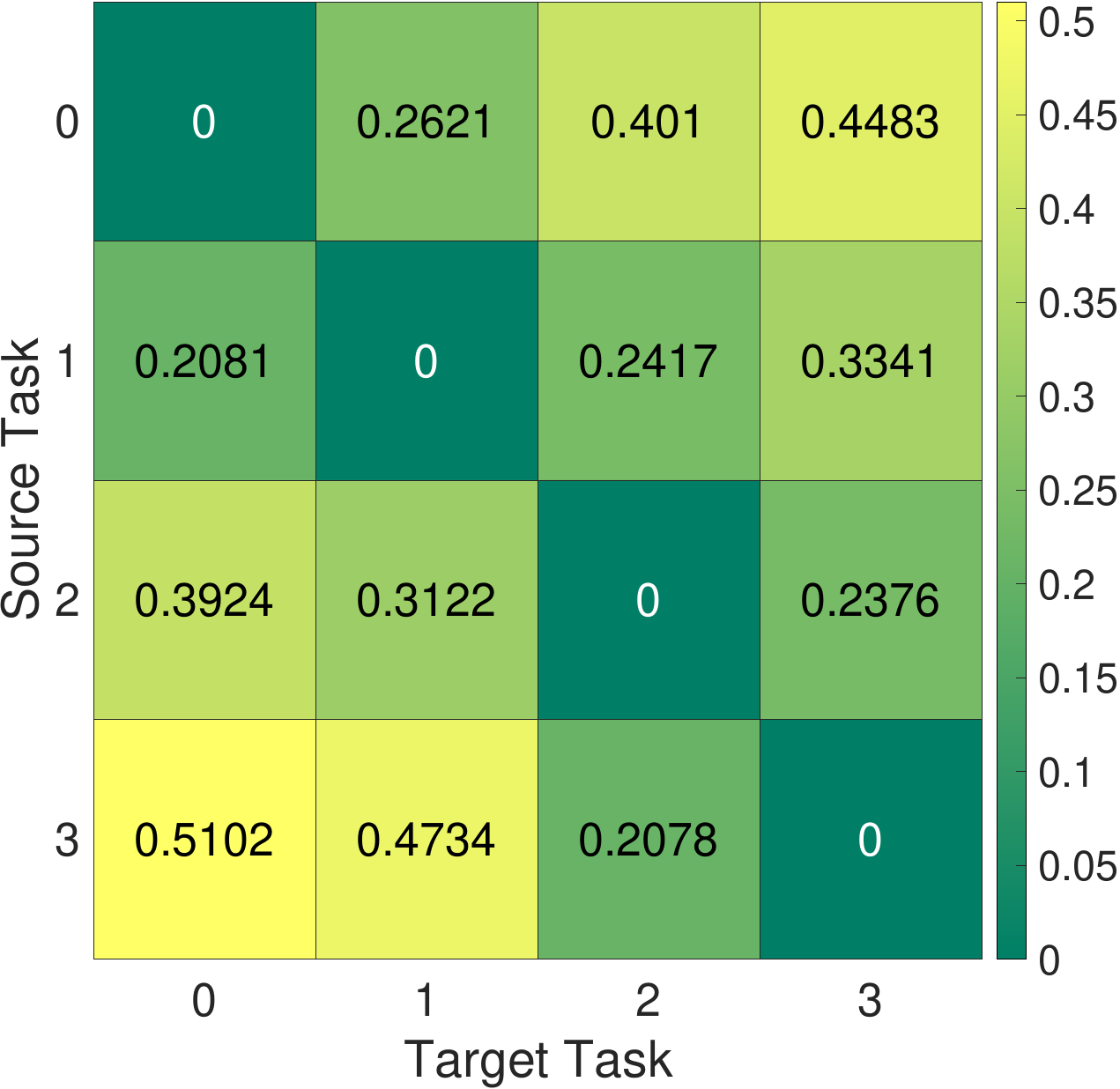}
    \end{subfigure}

    \begin{subfigure}{.33\textwidth}
        \centering
        \includegraphics[height=4.5cm]{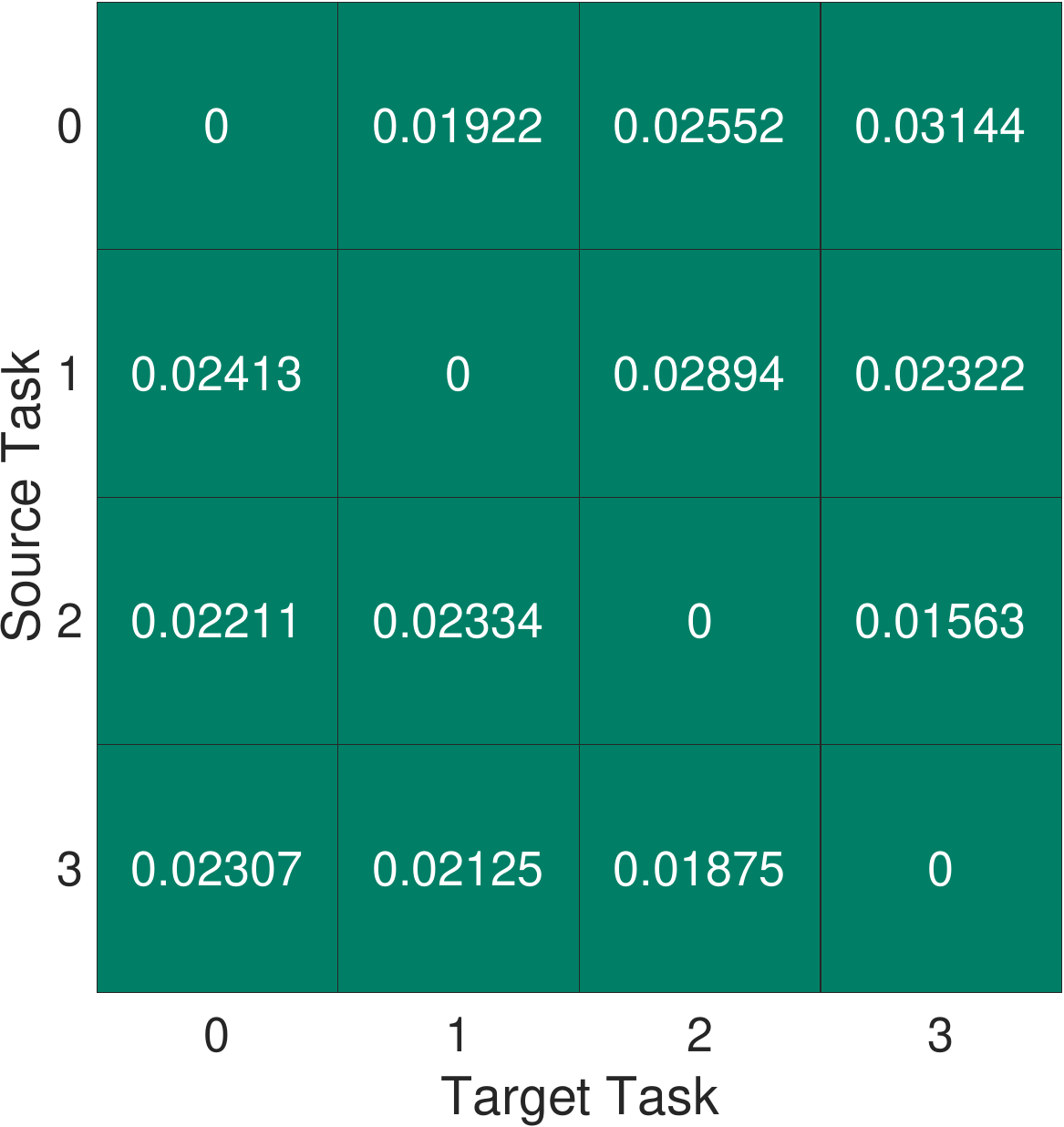}
        \caption{VGG-16}
    \end{subfigure}
    \begin{subfigure}{0.33\textwidth}
        \centering
        \includegraphics[height=4.5cm]{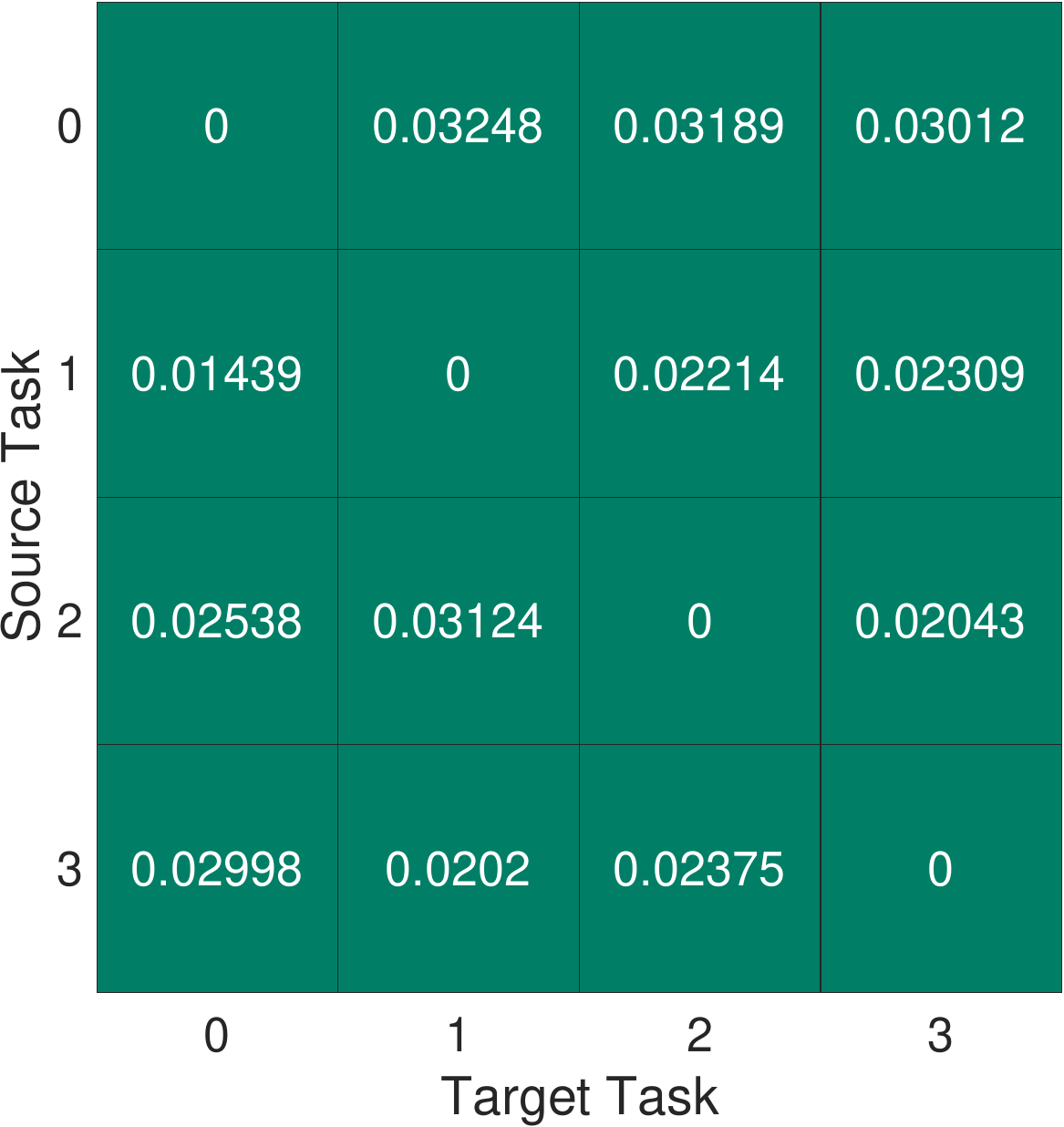}
        \caption{ResNet-18}
    \end{subfigure}
    \begin{subfigure}{.33\textwidth}
        \centering
        \includegraphics[height=4.5cm]{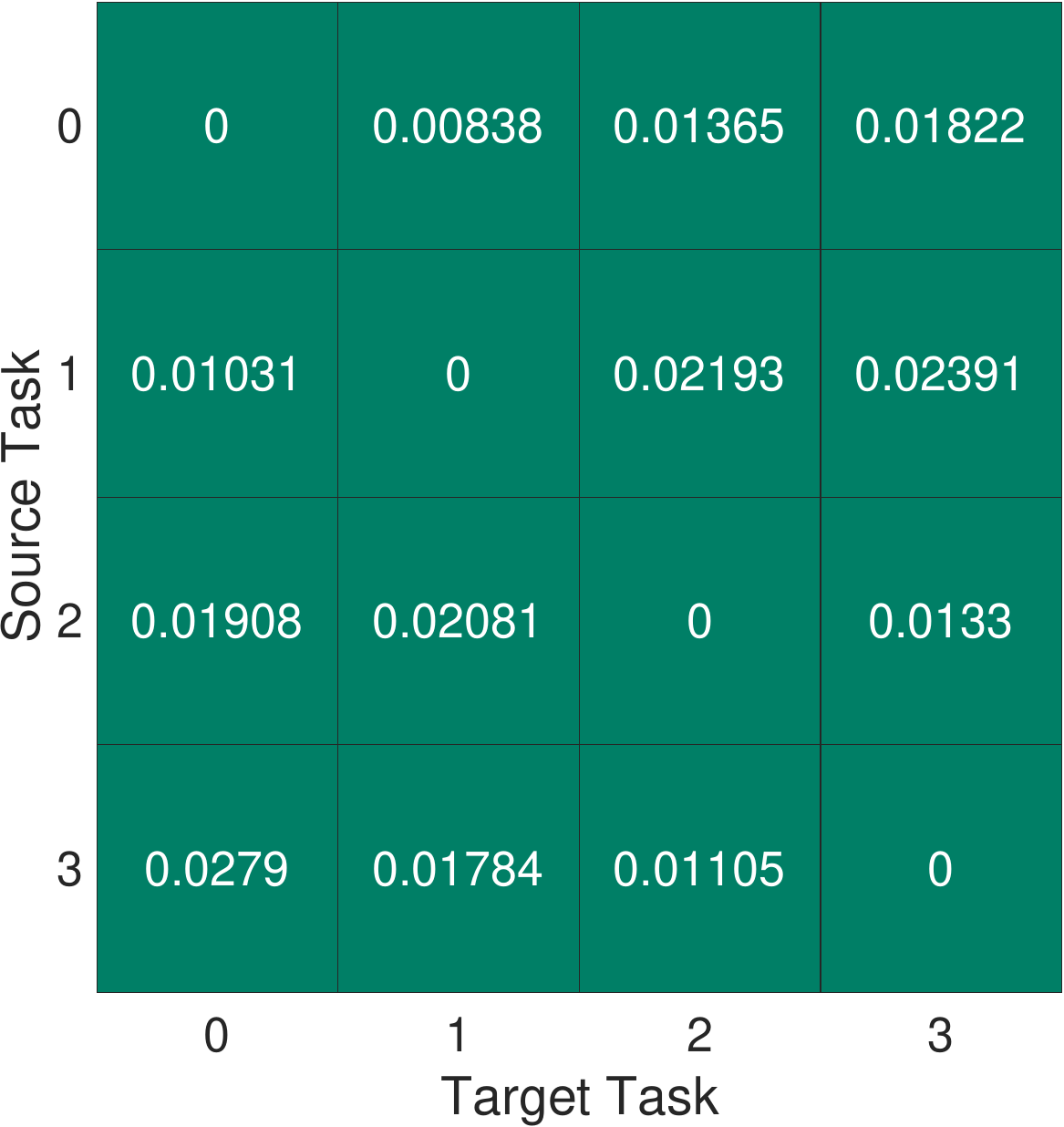}
        \caption{DenseNet-121}
    \end{subfigure}
    \caption{Distance from source tasks to the target tasks on ImageNet. The top row shows the mean values and the bottom row denotes the standard deviation of distances between classification tasks over 10 different trials.}
    \label{fig-imagenet}
\end{figure*}

\begin{figure*}
    \centering
    \begin{subfigure}{.33\textwidth}
        \centering
        \includegraphics[height=4.5cm]{mean-cifar10-vgg.pdf}
    \end{subfigure}
    \begin{subfigure}{.33\textwidth}
        \centering
        \includegraphics[height=4.5cm]{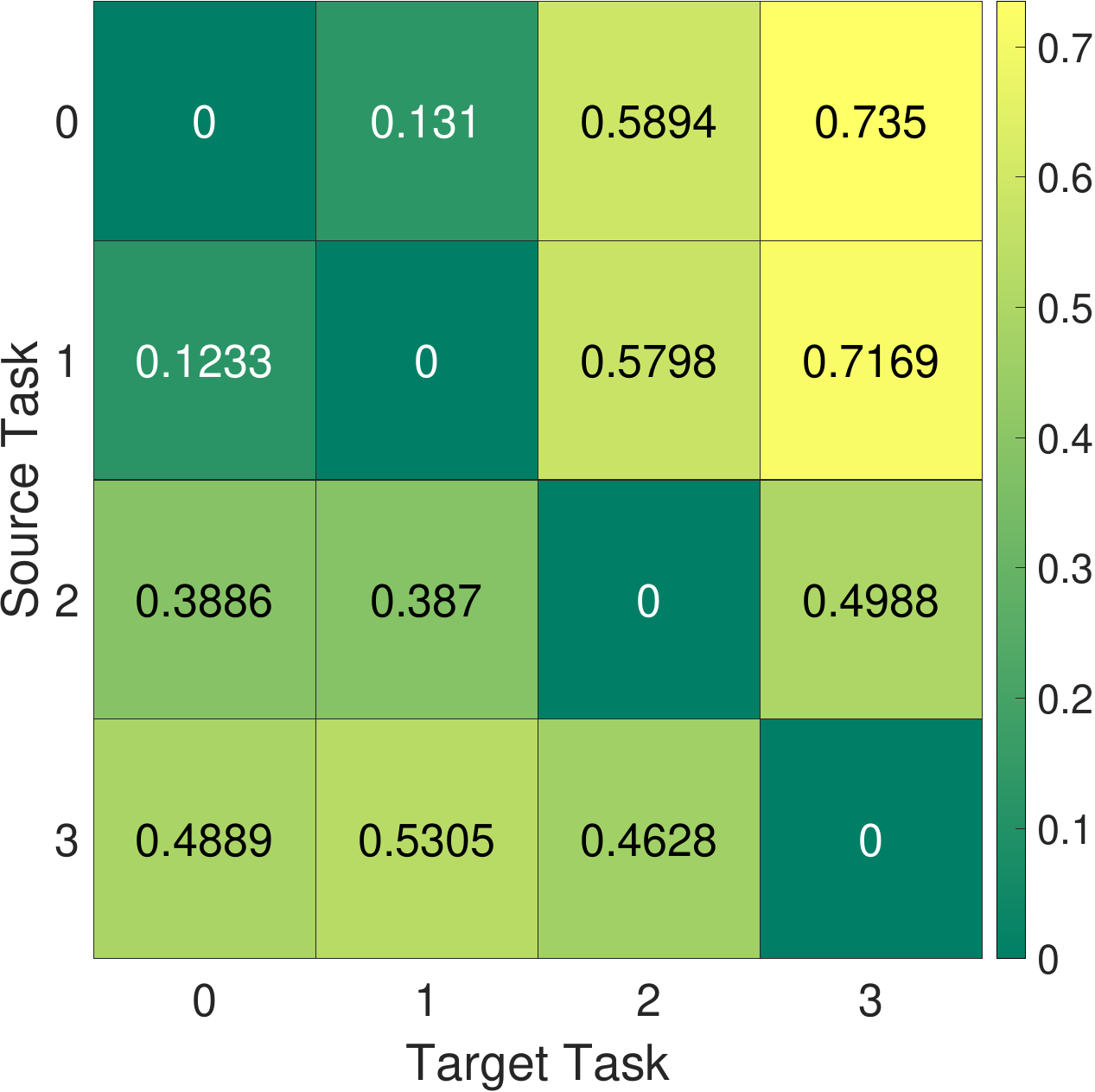}
    \end{subfigure}
    \begin{subfigure}{.33\textwidth}
        \centering
        \includegraphics[height=4.5cm]{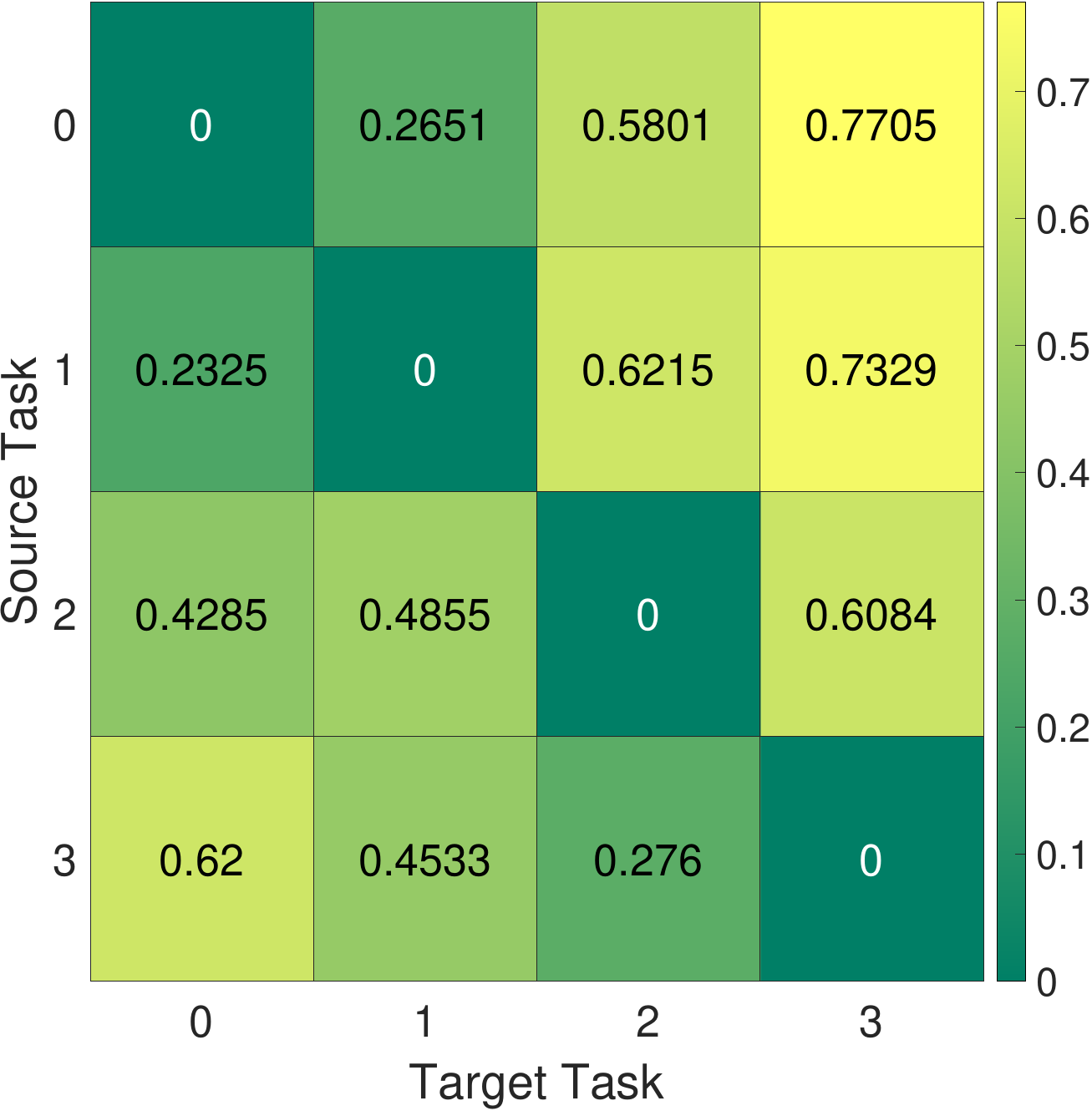}
    \end{subfigure}
    
    \begin{subfigure}{.33\textwidth}
        \centering
        \captionsetup{justification=centering}
        \includegraphics[height=4.5cm]{sig-cifar10-vgg.pdf}
        \caption{data augmentation}
    \end{subfigure}
    \begin{subfigure}{0.33\textwidth}
        \centering
        \captionsetup{justification=centering}
        \includegraphics[height=4.5cm]{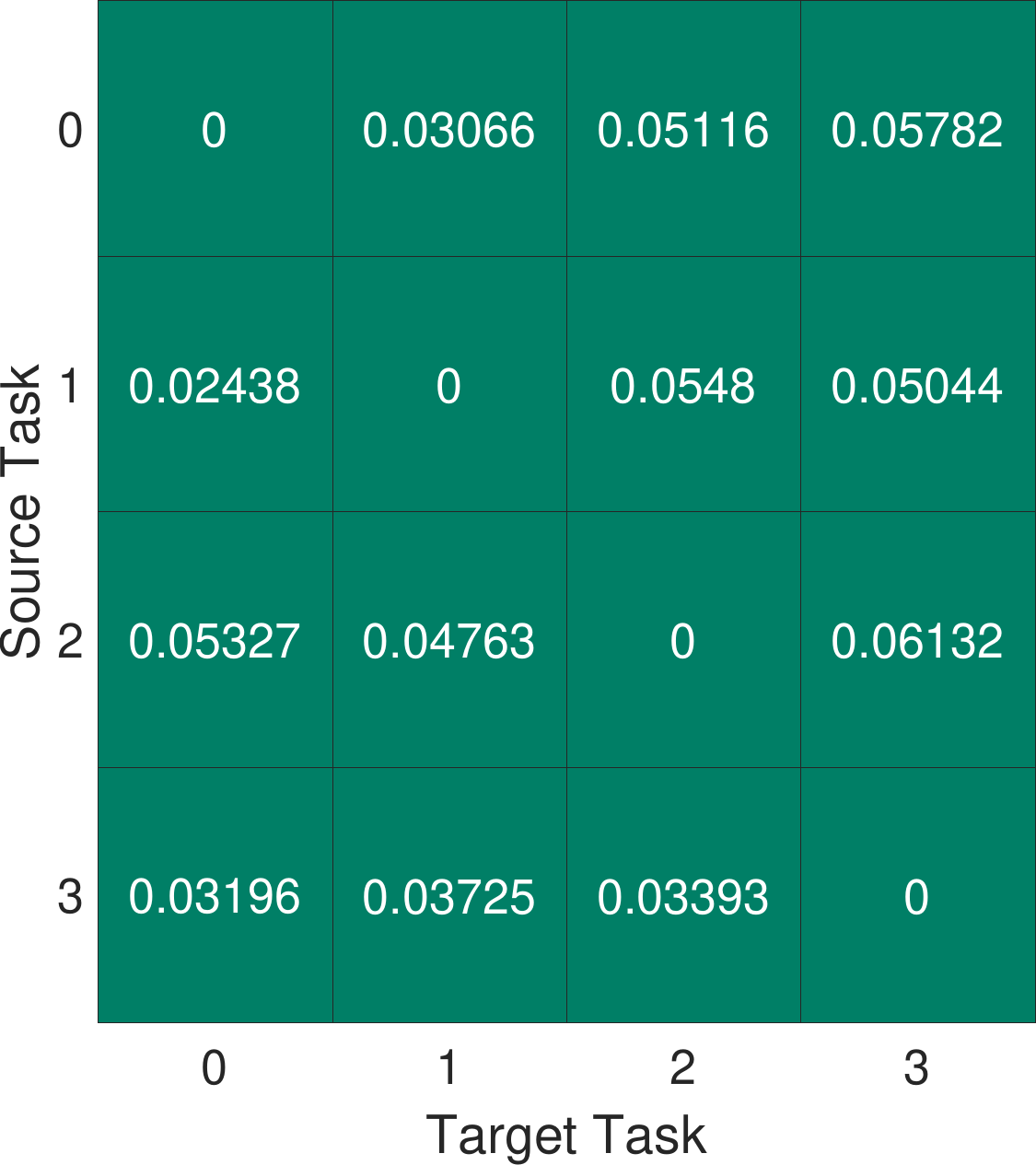}
        \caption{no data augmentation}
    \end{subfigure}
    \begin{subfigure}{.33\textwidth}
        \centering
        \captionsetup{justification=centering}
        \includegraphics[height=4.5cm]{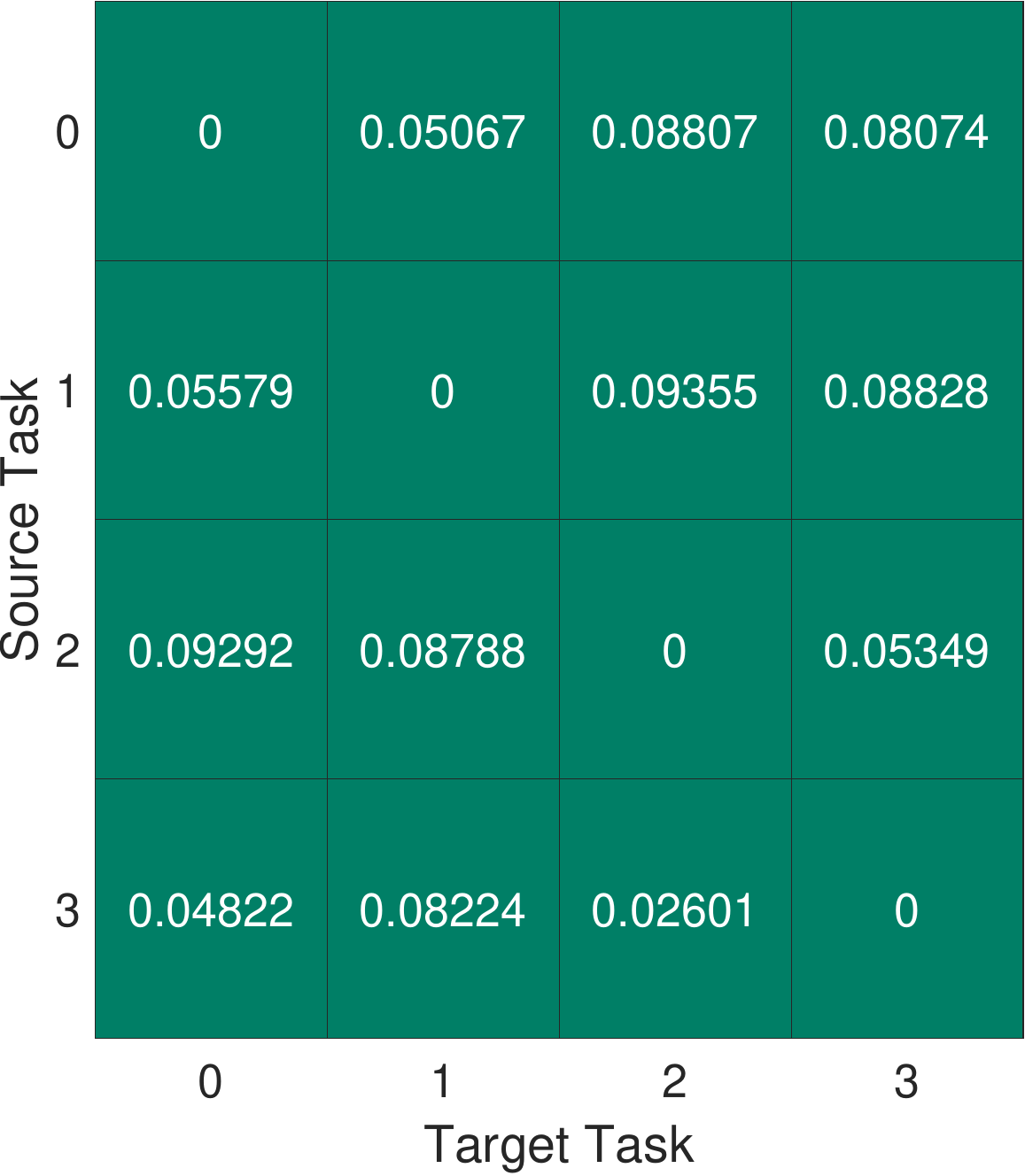}
        \caption{unbalanced data (full dataset)}
    \end{subfigure}
    \caption{The effect of different initial settings on computing distance between tasks defined on CIFAR-10 using VGG-16 as the $\varepsilon$-approximation network. The top row shows the mean values and the bottom row denotes the standard deviation of distances over 10 different trials. }
    \label{fig-cifar-10-b}
\end{figure*}

\subsection{Fisher Task Distance (FTD) Consistency}
\label{consistency}
In this experiment, we show the stability of the FTD by applying our distance on various classification tasks in MNIST, CIFAR-10, CIFAR-100, ImageNet, and Taskonomy datasets with different $\varepsilon$-approximation networks. For each dataset, we define $4$ classification tasks, all of which are variations of the full class classification task. For each task, we consider a balanced training dataset. That is, except for the classification tasks with all the labels, only a subset of the original training dataset is used such that the number of training samples across all the class labels to be equal. Additionally, we use $3$ widely-used and high-performance architectures as the $\varepsilon$-approximation networks, including VGG-16~\cite{simonyan2014very}, ResNet-18~\cite{he2016deep}, DenseNet-121~\cite{huang2017densely}. To make sure that our results are statistically significant, we run our experiments $10$ times with each of the $\varepsilon$-approximation networks being initialized with a different random seed each time and report the mean and the standard deviation of the computed distance.

\subsubsection{MNIST}
We define $4$ tasks on the MNIST dataset. Task $0$ and $1$ are the binary classification tasks of detecting digits $0$ and $6$, respectively. Task $2$ is a $5$-class classification of detecting digits $0, 1, 2, 3$, and anything else. Task $3$ is the full $10$ digits classification. Figure~\ref{fig-mnist} illustrates the mean and standard deviation of the distances between each pair of tasks after $10$ runs using $3$ different architectures. The columns of the tables denote the distance to the target task and the rows represent the distance from the source tasks. Our results suggest that Task $0$ and $1$ are highly related, and Task $3$ is the closest task to Task $2$. Moreover, the relation between tasks remain the same regardless of choosing $\varepsilon$-approximation networks.

\subsubsection{CIFAR-10}
We define $4$ tasks in the CIFAR-10 dataset. Task $0$ is a binary classification of indicating $3$ objects: automobile, cat, ship (i.e., the goal is to decide if the given input image consists of one of these three objects or not). Task $1$ is analogous to Task $0$ but with different objects: cat, ship, truck. Task $2$ is a $4$-class classification with labels bird, frog, horse, and anything else. Task $3$ is the standard $10$ objects classification. Figure~\ref{fig-cifar10} illustrates the mean and standard deviation of the distance between CIFAR-10 tasks over $10$ trial runs, using $3$ different architectures. As we can see in Figures~\ref{fig-cifar10}, the closest tasks to target tasks in all the tables always result in a unique task no matter what $\varepsilon$-approximation network we choose. Additionally, in Figure~\ref{fig-cifar-10-b} (in the appendix), we study the effect of different initial settings, such as training with/without data augmentation, or using unbalanced data set for the above $4$ tasks on the CIFAR-10 data set and using VGG-16 as the $\varepsilon$-approximation network. Again the same conclusion about the consistency of the FTD holds. 

\subsubsection{CIFAR-100}
We define $4$ tasks in the CIFAR-100 dataset, consisting of $100$ objects equally distributed in $20$ sub-classes, each sub-class has 5 objects. We define Task $0$ as a binary classification of detecting an object that belongs to vehicles $1$ and $2$ sub-classes or not (i.e., the goal is to decide if the given input image consists of one of these $10$ vehicles or not). Task $1$ is analogous to Task $0$ but with different sub-classes: household furniture and devices. Task $2$ is a multi-classification with $11$ labels defined on vehicles $1$, vehicles $2$, and anything else. Finally, Task $3$ is defined similarly to Task $2$; however, with the $21$-labels in vehicles $1$, vehicles $2$, household furniture, household device, and anything else. Figure~\ref{fig-cifar-100} illustrates the mean and the standard deviation of the distance between CIFAR-100 tasks after $10$ runs using $3$ different $\varepsilon$-approximation networks. Here, the closest tasks to any target tasks are distinctive regardless of the choice of the $\varepsilon$-approximation network.

\subsubsection{ImageNet}
Finally, we define four $10$-class classification tasks in ImageNet dataset. For each class, we consider $800$ for training and $200$ for the test samples. The list of $10$ classes in Task $0$ includes tench, English springer, cassette player, chain saw, church, French horn, garbage truck, gas pump, golf ball, parachute. Task $1$ is similar to Task $0$; however, instead of $3$ labels of tench, golf ball, and parachute, it has samples from the grey whale, volleyball, umbrella classes. In Task $2$, we also replace $5$ labels of grey whale, cassette player, chain saw, volleyball, umbrella in Task $0$ with another $5$ labels given by platypus, laptop, lawnmower, baseball, cowboy hat. Lastly, Task $3$ is defined as a $10$-class classification task with samples from the following classes: analog clock, candle, sweatshirt, birdhouse, ping-pong ball, hotdog, pizza, school bus, iPod, beaver. The mean and standard deviation tables of the distances between ImageNet tasks for $10$ trials with $3$ $\varepsilon$-approximation networks are illustrated in Figure~\ref{fig-imagenet}. Again, it is observed that the order of the distance between the source and target tasks remains the same independent of the $\varepsilon$-approximation networks.

Overall, although the value of task distance depends on the $\varepsilon$-approximation network, the trend remains the same across all $3$ architectures. In addition, the standard deviation values in the bottom row of the figures suggest that the computed task distance is stable as the fluctuations over the mean values do not show any overlap with each other.

\begin{figure*}[t]
    \centering
    \includegraphics[width=19cm]{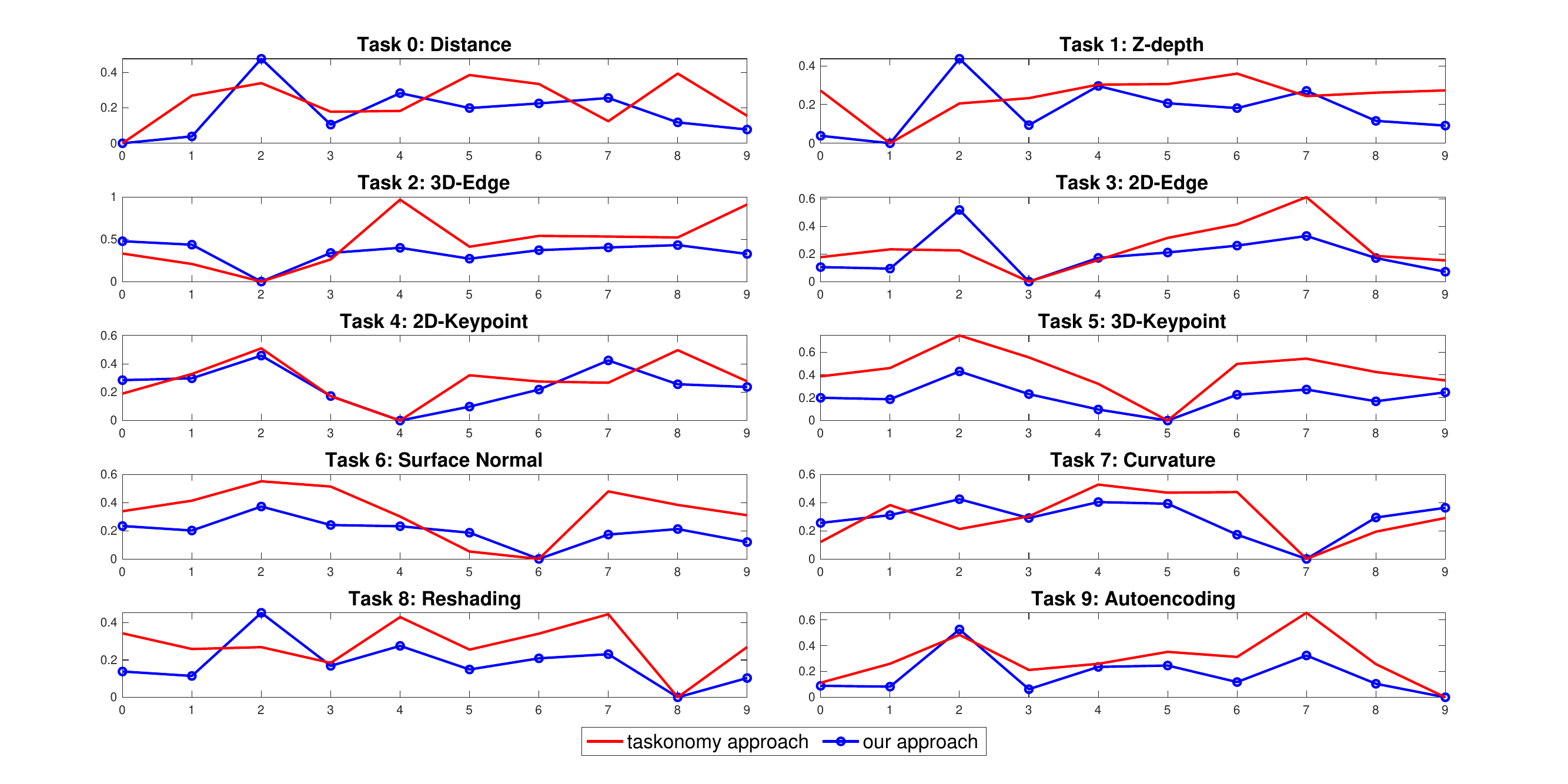}
    \caption{The comparison of task affinity between our approach and Taskonomy~\cite{zamir2018taskonomy} approach for each task.}
    \label{fig-taskonomy-compare}
\end{figure*}

\begin{figure}[t]
    \centering        
    \includegraphics[width=7.5cm]{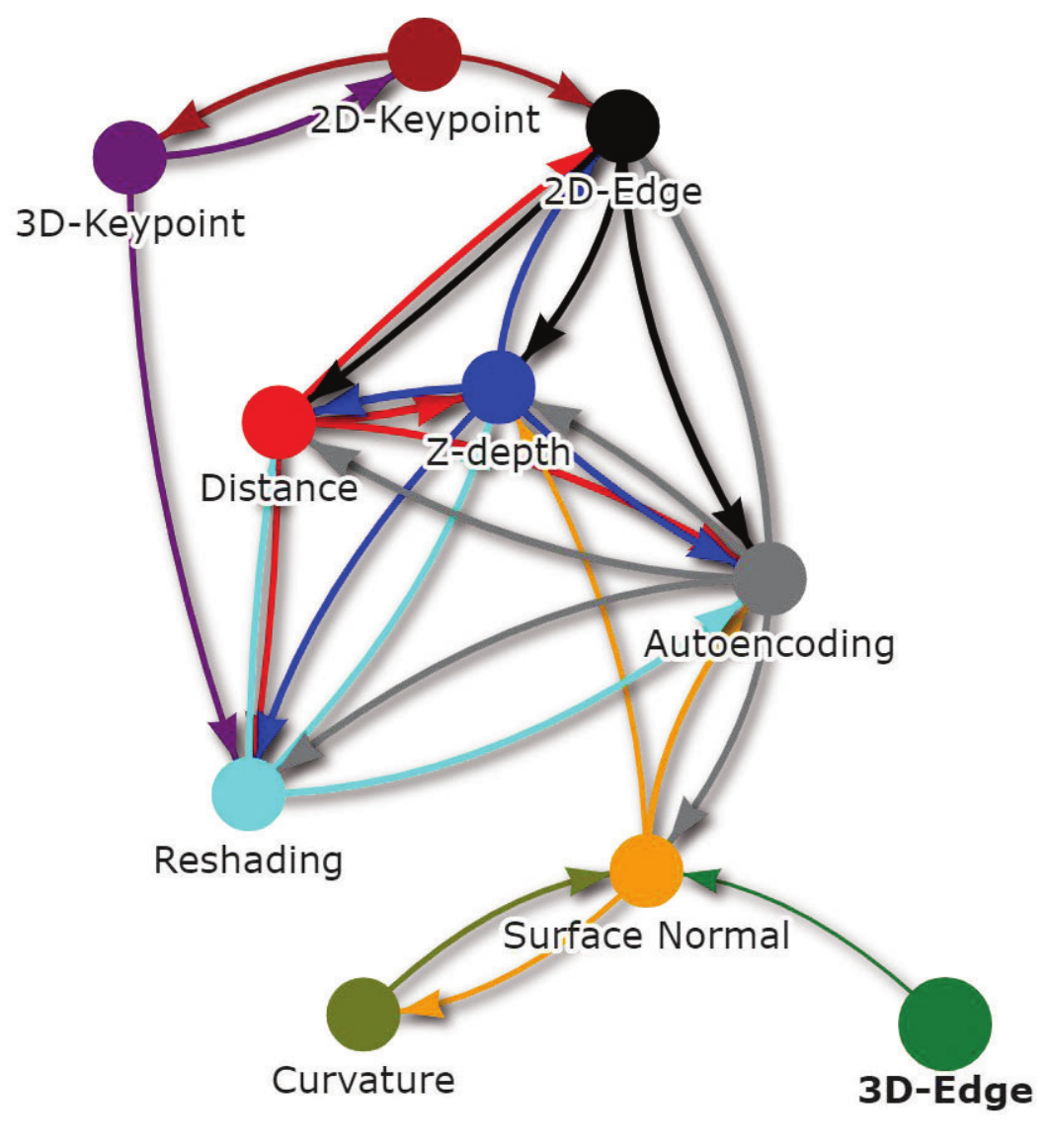}
    \caption{The atlas plot of tasks found from our approach indicates the computed relationship between tasks according to locations in space.}
    \label{fig-taskonomy-atlas}
\end{figure}

\begin{figure}
    \centering
    \begin{subfigure}{.45\textwidth}
        \centering
        \includegraphics[height=6.5cm]{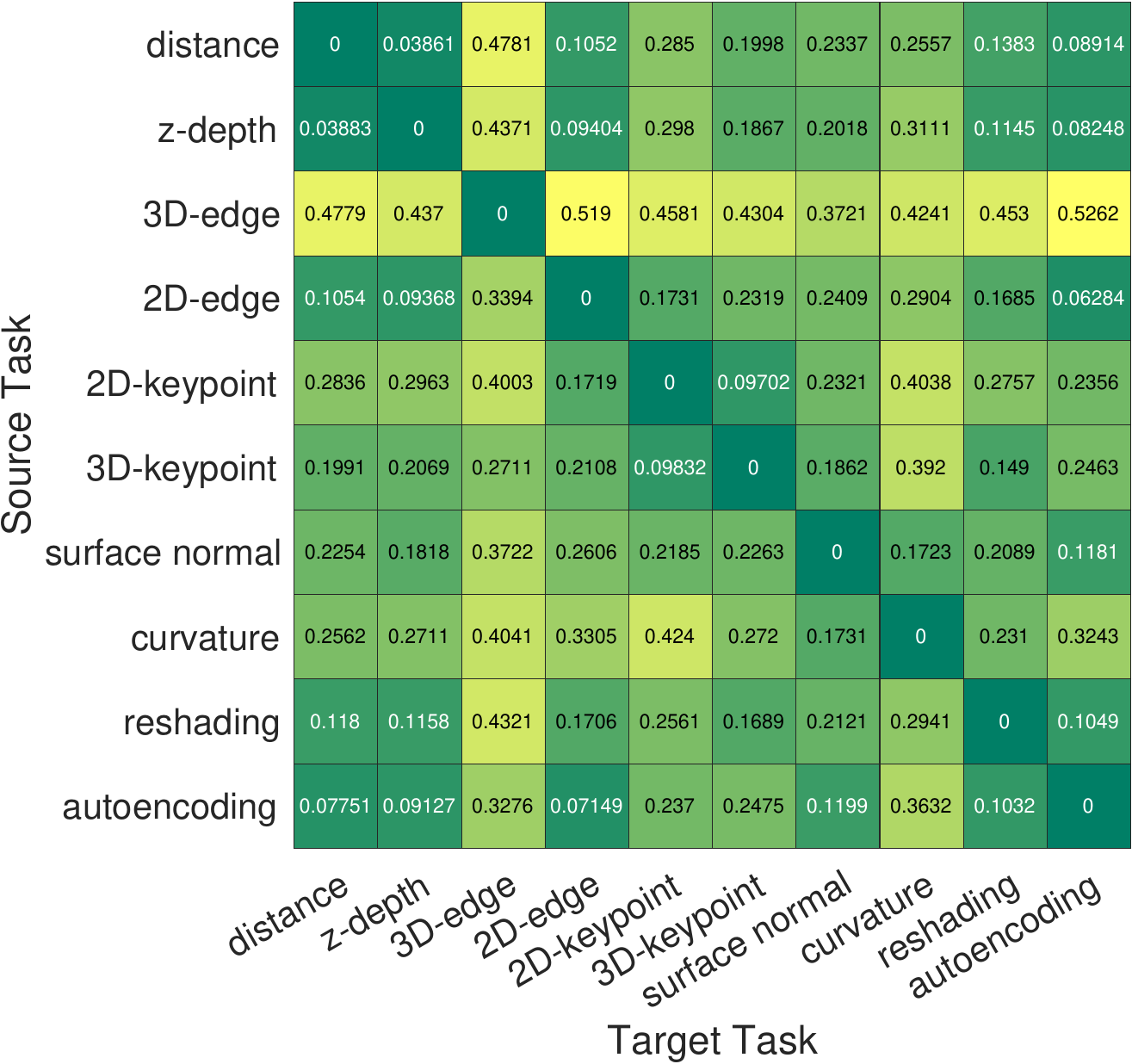}
    \end{subfigure}
    \begin{subfigure}{.45\textwidth}
        \centering
        \includegraphics[height=6.5cm]{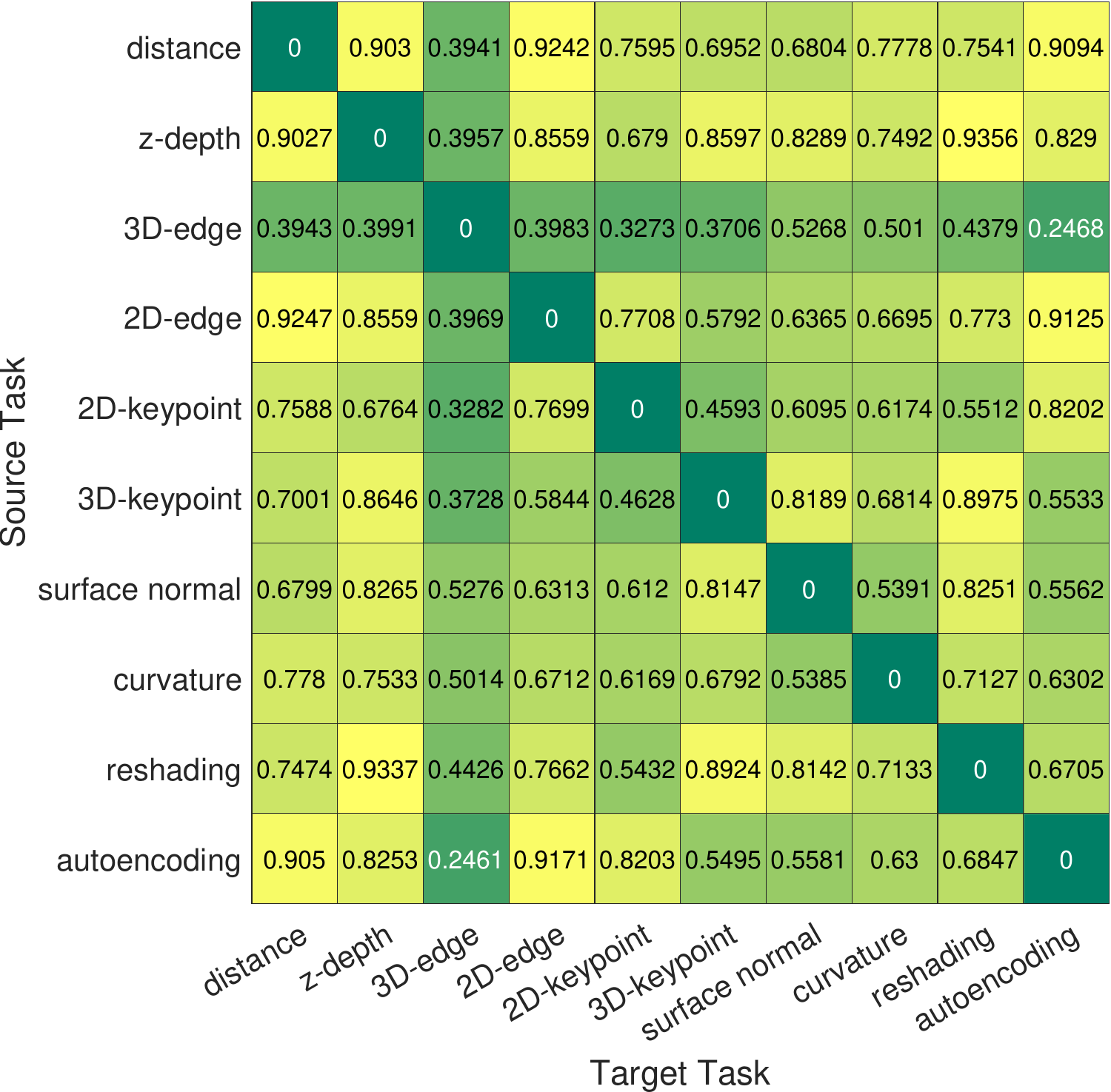}
    \end{subfigure}
    \begin{subfigure}{.45\textwidth}
        \centering
        \includegraphics[height=6.5cm]{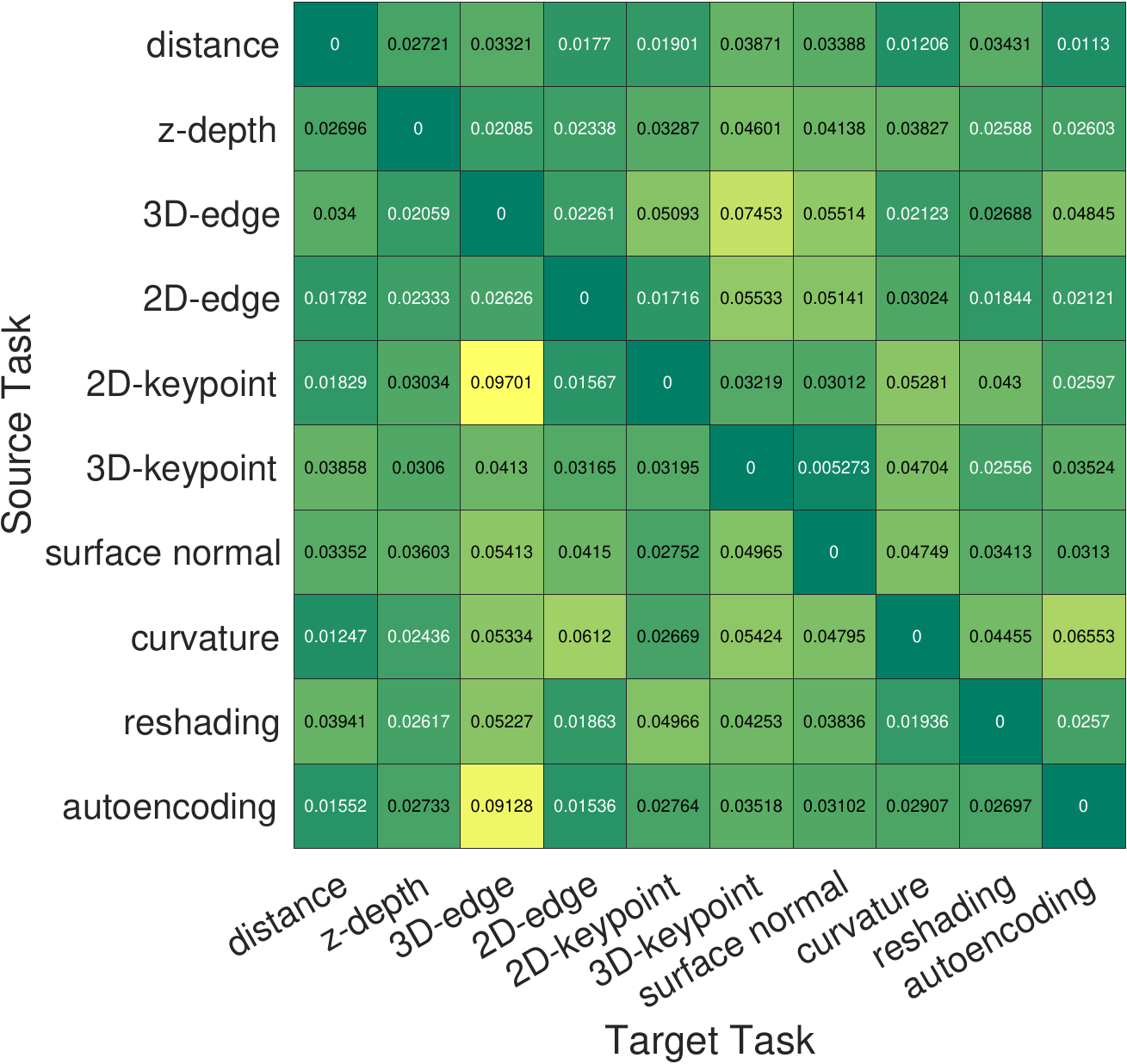}
    \end{subfigure}
    \caption{Comparison between FTD, cosine similarity, and the Taskonomy approach on tasks from Taskonomy dataset. The top panel shows the averaged distance found by our FTD approach over 10 different trials. The middle panel is the distances obtained by cosine similarity of the Fisher Information matrices. The bottom panel shows the task affinity found by brute-force approach~\cite{zamir2018taskonomy} after a single run.}
    \label{fig-taskonomy}
\end{figure}

\subsubsection{Taskonomy}
Here, we apply the FTD to the Taskonomy dataset and compare the task affinity found using our task distance with the brute-force method proposed by ~\cite{zamir2018taskonomy}, and the cosine similarity on the Taskonomy dataset. The Taskonomy dataset is a collection of $512\times 512$ colorful images of varied indoor scenes. It provides the pre-processed ground truth for $25$ vision tasks including semantic and low-level tasks. In this experiment, we consider a set of $10$ visual tasks, including: (0) Euclidean distance, (1) z-depth, (2) 3D-edge, (3) 2D-edge (4) 2D-keypoint, (5) 3D-keypoint, (6) surface normal, (7) curvature, (8) reshading, (9) autoencoding. Please see ~\cite{zamir2018taskonomy} for detailed task descriptions. Each task has $40,000$ training samples and $10,000$ test samples. A deep autoencoder architecture, including convolutional and linear layers, with a total of $50.51$ Mil parameters, is chosen to be the $\varepsilon$-approximation network for all visual tasks. 

In order to use the autoencoder for all the visual tasks without architecture modification, we convert all of the ground truth outputs to three channels. The top panel of Figure~\ref{fig-taskonomy} shows the mean of our task distance between each pair of tasks over $10$ different initial settings in the training of the $\varepsilon$-approximation network. The middle panel indicates the distance founded using the cosine similarity between a pair of Fisher Information matrices. The bottom panel shows the task affinity achieved by the brute-force method from the Taskonomy paper for only a single run. Note that, the task affinities are asymmetric (or non-commutative) in these approaches. As shown in Figure~\ref{fig-taskonomy}, our approach and the brute-force approach have a similar trend in identifying the closest tasks for a given incoming task. The task affinity found by the brute-force approach, however, requires a lot more computations, and does not give a clear boundary between tasks (e.g., difficult to identify the closest tasks to some target tasks). Our approach, on the other hand, is statistical and  determines a clear boundary for the identification of related tasks based on the distance. Lastly, the cosine similarity distance between Fisher Information matrices differs from both our approach and the brute-force approach in several tasks, while not showing any clear boundary between tasks. Consequently, the cosine similarity with Fisher Information matrices is not suitable for this dataset. Additionally, Figure~\ref{fig-taskonomy-compare} illustrates the comparison of task affinity by our approach and by the brute-force approach in the Taskonomy paper. We note that both approaches follow a similar trend for most of the tasks. Additionally, Figure~\ref{fig-taskonomy-atlas} shows the atlas plot of tasks found by our approach, which represents the relationship of tasks according to the location in space. Overall, our FTD is capable of identifying the related tasks with clear statistical boundaries between tasks (i.e., low standard deviation) while requiring significantly less computational resources compared to the brute-force approach (with less statistical significance) in the Taskonomy paper.

\subsection{Application in Neural Architecture Search}

In this experiment, we show the application of the FTD in Neural Architecture Search (NAS) by utilizing the computed distances between tasks from the various classification tasks in MNIST, CIFAR-10, CIFAR-100, and ImageNet datasets, and perform the architecture search.

\subsubsection{MNIST} We consider the problem of learning architecture for the target Task $2$ of MNIST dataset, using the other aforementioned tasks as our baseline tasks. It is observed in Figure~\ref{fig-mnist} that Task $3$ is the closest one to Task $2$. As the result, we apply cell structure and the operations of Task $3$ to generate a suitable search space for the target task. The results in Table~\ref{table-mnist} show the best test accuracy of the optimal architecture found by our method compared to well-known handcrafted networks (i.e., VGG-16, ResNet-18, DenseNet-121), the state-of-the-art NAS methods (i.e., random search algorithm~\cite{li2020random}, ENAS~\cite{pham2018efficient}, DARTS~\cite{liu2018darts}, PC-DARTS~\cite{xu2019pc}, TE-NAS~\cite{chen2021neural}). The architecture discovered by our method is competitive with these networks while it results in a significantly smaller amount of parameters and GPU days. 

\begin{table}[t]
\caption{\label{table-mnist} Comparison of our TA-NAS framework with the hand-designed image classifiers, and state-of-the-art NAS methods on Task $2$ (binary classification) of MNIST.}
\begin{center}
\begin{tabular}{l|cc|c}
\hline
\multicolumn{1}{l}{\bf Architecture} &\multicolumn{1}{c}{\bf Accuracy} &\multicolumn{1}{c}{\bf No. Params.} &\multicolumn{1}{c}{ \bf GPU}\\

\multicolumn{1}{l}{}  &\multicolumn{1}{c}{} &\multicolumn{1}{c}{\bf (Mil)} &\multicolumn{1}{c}{ \bf days}\\
\hline
VGG-16                  & 99.41     & 14.72     & -\\ 
ResNet-18               & 99.47     & 11.44     & -\\ 
DenseNet-121            & 99.61     & 6.95      & -\\
\hline
Random Search           & 99.52     & 2.12      & 5\\ 
ENAS (1st)              & 94.29     & 4.60      & 2\\
ENAS (2nd)              & 94.71     & 4.60      & 4\\
DARTS (1st)             & 98.82     & 2.17      & 2\\
DARTS (2nd)             & 99.44     & 2.23      & 4\\
PC-DARTS (1st)          & 98.76     & 1.78      & 2\\
PC-DARTS (2nd)          & 99.88     & 2.22      & 4\\
TE-NAS                  & 99.71     & 2.79      & 2\\
\hline
\textbf{TA-NAS (ours)}  & \textbf{99.86}     & \textbf{2.14}      & \textbf{2}\\  
\hline
\end{tabular}
\end{center}
\end{table}

\subsubsection{CIFAR-10 \& CIFAR-100} We consider the problem of searching for a high-performing and efficient architecture for Task $2$ in CIFAR-10, and Task $2$ in CIFAR-100 datasets. As observed in Figure~\ref{fig-cifar10} and Figure~\ref{fig-cifar-100}, we consider Task $3$ as the closest task for both cases. Results in Table~\ref{table-cifar10} and~\ref{table-cifar100} suggest that our constructed architectures for these tasks have higher test accuracy with a fewer number of parameters and GPU days compared to other approaches. The poor performance of ENAS in CIFAR-100 highlights the lack of robustness of this method since its search space is defined to only for full class classification in CIFAR-10.

\begin{table}[t]
\caption{\label{table-cifar100} Comparison of our TA-NAS framework with the hand-designed image classifiers, and state-of-the-art NAS methods on  Task $2$ ($11$-class classification) of CIFAR-100.}
\begin{center}
\begin{tabular}{l|cc|c}
\hline
\multicolumn{1}{l}{\bf Architecture} &\multicolumn{1}{c}{\bf Accuracy} &\multicolumn{1}{c}{\bf No. Params.} &\multicolumn{1}{c}{ \bf GPU}\\

\multicolumn{1}{l}{} &\multicolumn{1}{c}{} &\multicolumn{1}{c}{\bf (Mil)} &\multicolumn{1}{c}{ \bf days}\\
\hline
VGG-16                  & 83.93     & 14.72     & -\\ 
ResNet-18               & 84.56     & 11.44     & -\\ 
DenseNet-121            & 88.47     & 6.95      & -\\
\hline
Random Search           & 88.55         & 3.54         & 5\\ 
ENAS                    & 10.49         & 4.60         & 4\\
DARTS                   & 87.57         & 3.32         & 4\\
PC-DARTS                & 85.36         & 2.43         & 4\\
TE-NAS                  & 88.92         & 3.66         & 4\\
\hline
\textbf{TA-NAS (ours)}  & \textbf{90.96}     & \textbf{3.17}      & \textbf{4}\\ 
\hline
\end{tabular}
\end{center}
\end{table}

\begin{table}[t]
\caption{\label{table-cifar10} Comparison of our TA-NAS framework with the hand-designed image classifiers, and state-of-the-art NAS methods on Task $2$ ($4$-class classification) of CIFAR-10.}
\begin{center}
\begin{tabular}{l|cc|c}
\hline
\multicolumn{1}{l}{\bf Architecture} &\multicolumn{1}{c}{\bf Accuracy} &\multicolumn{1}{c}{\bf No. Params.} &\multicolumn{1}{c}{ \bf GPU}\\

\multicolumn{1}{l}{} &\multicolumn{1}{c}{} &\multicolumn{1}{c}{\bf (Mil)} &\multicolumn{1}{c}{ \bf days}\\
\hline
VGG-16                  & 86.75     & 14.72     & -\\ 
ResNet-18               & 86.93     & 11.44     & -\\ 
DenseNet-121            & 88.12     & 6.95      & -\\
\hline
Random Search           & 88.55     & 3.65      & 5\\ 
ENAS (1st)              & 73.23     & 4.60      & 2\\
ENAS (2nd)              & 75.22     & 4.60      & 4\\
DARTS (1st)             & 90.11     & 3.12      & 2\\
DARTS (2nd)             & 91.19     & 3.28      & 4\\
PC-DARTS (1st)          & 92.07     & 3.67      & 2\\
PC-DARTS (2nd)          & 92.49     & 3.66      & 4\\
TE-NAS                  & 91.02     & 3.78      & 2\\
\hline
\textbf{TA-NAS (ours)}  & \textbf{92.58}     & \textbf{3.13}      & \textbf{2}\\ 
\hline
\end{tabular}
\end{center}
\end{table}

\subsubsection{ImageNet} We consider Task $1$ in ImageNet dataset as the target task. Based on the computed distances in Figure~\ref{fig-imagenet}, we use Task $0$ as the closest source task to our target task. Table~\ref{table-imagenet} presents results indicating that our model has higher test accuracy with a fewer number of parameters compared to other approaches.  In complex datasets (e.g., CIFAR-100, ImageNet), the method with fixed search space (i.e., ENAS) is only capable of finding architectures for tasks in standard datasets, performs poorly compared with other methods in our benchmark. Our experiments suggest that the proposed framework can utilize the knowledge of the most similar task in order to find a high-performing architecture for the target task with a fewer number of parameters.

\begin{table}[t]
\caption{\label{table-imagenet} Comparison of our TA-NAS framework with the hand-designed image classifiers, and state-of-the-art NAS methods on Task $1$ ($10$-class classification) of ImageNet.}
\begin{center}
\begin{tabular}{l|cc|c}
\hline
\multicolumn{1}{l}{\bf Architecture} &\multicolumn{1}{c}{\bf Accuracy} &\multicolumn{1}{c}{\bf No. Params.} &\multicolumn{1}{c}{ \bf GPU}\\

\multicolumn{1}{l}{} &\multicolumn{1}{c}{} &\multicolumn{1}{c}{\bf (Mil)} &\multicolumn{1}{c}{ \bf days}\\
\hline
VGG-16                  & 89.88     & 14.72     & -\\ 
ResNet-18               & 91.14     & 11.44     & -\\ 
DenseNet-121            & 94.76     & 6.95      & -\\
\hline
Random Search           & 95.02     & 3.78      & 5\\ 
ENAS                    & 33.65     & 4.60      & 4\\
DARTS                   & 95.22     & 3.41      & 4\\
PC-DARTS                & 88.00     & 1.91      & 4\\
TE-NAS                  & 95.37     & 4.23      & 4\\
\hline
\textbf{TA-NAS (ours)}  & \textbf{95.92}     & \textbf{3.43}      & \textbf{4}\\
\hline
\end{tabular}
\end{center}
\end{table}

\section{Conclusions}
A task similarity measure based on the Fisher Information matrix has been introduced in this paper. This non-commutative measure called Fisher Task distance (FTD), represents the complexity of applying the knowledge of one task to another. The theory and experimental experiments demonstrate that the distance is consistent and well-defined. In addition, two  applications of FTD, including transfer learning and NAS has been investigated. In particular, the task affinity results found using this measure is well-aligned with the result found using the traditional transfer learning approach. Moreover, the FTD is applied in NAS to define a reduced search space of architectures for a target task. This reduces the complexity of the architecture search, increases its efficiency, and leads to superior performance with a smaller number of parameters. 

\section{Appendix}
Here, we provide the the proof of the proposition~\ref{proposition1}, the proof of the theorem~\ref{theorem1}, and the proof of the theorem~\ref{theorem2}.

\begin{proposition1}
Let $X$ be the dataset for the target task $T$. For any pair of structurally-similar $\varepsilon$-approximation network w.r.t $(T,X)$ using the full or stochastic gradient descent algorithm with the same initialization settings, learning rate, and the same order of data batches in each epoch for the SGD algorithm, the Fisher task distance between the above pair of $\varepsilon$-approximation networks is always zero.
\end{proposition1}
\begin{proof}[\textbf{Proof of Proposition \ref{proposition1}}]
Let $N_1$ and $N_{2}$ be two structurally-similar $\varepsilon$-approximation network w.r.t $(T,X)$ trained using the full or stochastic gradient descent algorithm. According to the Definition~\ref{SSappNetTX} and assumptions in the proposition, the Fisher Information Matrices of $N_1$ and $N_{2}$ are the same; hence, the Fisher task distance is zero.
\end{proof}

\begin{theorem1}
Let $X$ be the dataset for the target task $T$. Consider $N_1$ and $N_2$ as two structurally-similar $\varepsilon$-approximation networks w.r.t. $(T,X)$ respectively with the set of weights $\theta_1$ and $\theta_2$ trained using the SGD algorithm where a diminishing learning rate is used for updating weights. Assume that the loss function $L$ for the task $T$ is strongly convex, and its 3rd-order continuous derivative exists and bounded. Let the noisy gradient function in training $N_1$ and $N_2$ networks using SGD algorithm be given by:
\begin{equation}
    g({\theta_i}_t, {\epsilon_i}_t) = \nabla L({\theta_i}_t) + {\epsilon_i}_t, \ \ for\ \ i=1,2,
\end{equation}
where ${\theta_i}_t$ is the estimation of the weights for network $N_i$ at time $t$, and $\nabla L({\theta_i}_{t})$ is the true gradient at ${\theta_i}_t$. Assume that ${\epsilon_i}_t$  satisfies $\mathbb{E}[{\epsilon_i}_t|{\epsilon_i}_0,...,{\epsilon_i}_{t-1}] = 0$, and satisfies $\displaystyle s = \lim_{t\xrightarrow[]{}\infty} \big|\big|[{\epsilon_i}_t {{\epsilon_i}_t}^T | {\epsilon_i}_0,\dots,{\epsilon_i}_{t-1}]\big|\big|_{\infty}<\infty$ almost surely (a.s.). Then the Fisher task distance between $N_1$ and $N_2$ computed on the average of estimated weights up to the current time $t$ converges to zero as $t \rightarrow \infty$. That is,
\begin{align}
    d_t = \frac{1}{\sqrt{2}}\Big|\Big|\Bar{F_1}_t^{1/2} - \Bar{F_2}_t^{1/2}\Big|\Big|_F \xrightarrow[]{\mathcal{D}}0,
\end{align}
where $\Bar{F}_{i_{t}} = F(\bar{\theta}_{i_{t}})$ with $\bar{\theta}_{i_{t}} = \frac{1}{t}\sum_t \theta_{i_{t}}$, for $i=1,2$.
\end{theorem1}

\begin{proof}[\textbf{Proof of Theorem \ref{theorem1}}] Here, we show the proof for the full Fisher Information Matrix; however, the same results holds for the diagonal approximation of the Fisher Information Matrix. Let $N_1$ with weights $\theta_1$ and $N_2$ with weights $\theta_2$ be the two structurally-similar $\varepsilon$-approximation networks w.r.t. $(T,X)$. Let $n$ be the number of trainable parameters in $N_1$ and $N_2$. Since the objective function is strongly convex and the fact that $N_1$ and $N_2$ are structurally-similar $\varepsilon$-approximation networks w.r.t. $(T,X)$, both of these network will obtain the optimum solution $\theta^*$ after training a certain number of epochs with stochastic gradient descend. By the assumption on the conditional mean of the noisy gradient function and the assumption on  $S$, the conditional covariance matrix is finite as well, i.e., $C=\lim_{t\xrightarrow[]{}\infty} \mathbb{E}[{\epsilon_i}_t {{\epsilon_i}_t}^T | {\epsilon_i}_0,\dots,{\epsilon_i}_{t-1}]<\infty$; hence, we can invoke the following result due to Polyak et al. \cite{Polyak1992AccelerationOS}:
\begin{equation}\label{eq5}
    \sqrt{t}(\Bar{\theta}_t - \theta^*) \xrightarrow[]{\mathcal{D}} \mathcal{N} \Big( 0, \mathbf{H}\big(L(\theta^*)\big)^{-1} C \mathbf{H}^T\big(L(\theta^*)\big)^{-1} \Big),
\end{equation}
as $t \xrightarrow[]{} \infty$. Here, $\mathbf{H}$ is Hessian matrix, $\theta^*$ is the global minimum of the loss function, and $\Bar{\theta_t} = \frac{1}{t} \sum_t \theta_t$. Hence, for networks $N_1$ and $N_2$ and from Equation (\ref{eq5}), $\sqrt{t}(\Bar{\theta_1}_t - \theta^*)$ and $\sqrt{t}(\Bar{\theta_2}_t - \theta^*)$ are asymptotically normal random vectors:
\begin{align}
    \sqrt{t}(\Bar{\theta_1}_t - \theta^*) \xrightarrow[]{\mathcal{D}} \mathcal{N}(0, \Sigma_1), \label{eq7}\\
    \sqrt{t}(\Bar{\theta_2}_t - \theta^*) \xrightarrow[]{\mathcal{D}} \mathcal{N}(0, \Sigma_2), \label{eq8}
\end{align}
where $\Sigma_1 = \mathbf{H}\big(L(\theta^*)\big)^{-1} C_1 \mathbf{H}^T\big(L(\theta^*)\big)^{-1}$, and $\Sigma_2 = \mathbf{H}\big(L(\theta^*)\big)^{-1} C_2 \mathbf{H}^T\big(L(\theta^*)\big)^{-1}$. The Fisher Information $F(\theta)$ is a continuous and differentiable function of $\theta$. Since it is also a positive definite matrix, $F(\theta)^{1/2}$ is well-defined. Hence, by applying the Delta method to Equation (\ref{eq7}), we have:
\begin{equation}\label{eq9}
    \sqrt{t}(\Bar{F_1}_t^{1/2} - {F^*}^{1/2}) \xrightarrow[]{\mathcal{D}} \mathcal{N}(0, \Sigma_1^*),
\end{equation}
where $\Bar{F}_{1_t} = F(\bar{\theta}_{1_t})$, and the covariance matrix $\Sigma_1^*$ is given by $\Sigma_1^* = \mathbf{J}_{\theta} \Big(\mathbf{vec} \big( F(\theta^*)^{1/2} \big) \Big) \Sigma_1 \mathbf{J}_{\theta} \Big( \mathbf{vec} \big( F(\theta^*)^{1/2} \big) \Big)^T$. Here, 
$\mathbf{vec}()$ is the vectorization operator, $\theta^*$ is a $n \times 1$ vector of the optimum parameters, $F(\theta^*)$ is a $n \times n$ Matrix evaluated at the minimum, and  $\mathbf{J}_{\theta}(F(\theta^*))$ is a $n^2 \times n$ Jacobian matrix of the Fisher Information Matrix. Similarly, from Equation (\ref{eq8}), we have:
\begin{equation}\label{eq10}
    \sqrt{t}(\Bar{F_2}_t^{1/2} - {F^*}^{1/2}) \xrightarrow[]{\mathcal{D}} \mathcal{N}(0, \Sigma_2^*),
\end{equation}
where $\Sigma_2^* = \mathbf{J}_{\theta} \Big( \mathbf{vec} \big( F(\theta^*)^{1/2} \big) \Big) \Sigma_2 \mathbf{J}_{\theta} \Big( \mathbf{vec} \big (F(\theta^*)^{1/2} \big) \Big)^T$. As a result, $(\Bar{F_1}_t^{1/2} - \Bar{F_2}_t^{1/2})$ is asymptotically a normal random vector:
\begin{equation}
    (\Bar{F_1}_t^{1/2} - \Bar{F_2}_t^{1/2}) \xrightarrow[]{\mathcal{D}} \mathcal{N} \Big (0,V_1 \Big).
\end{equation}
where $V_1 = \frac{1}{t}(\Sigma_1^* + \Sigma_2^*)$. As $t$ approaches infinity, $\displaystyle \frac{1}{t}(\Sigma_1^* + \Sigma_2^*) \xrightarrow[]{} 0$. As a result,    $d_t=\frac{1}{\sqrt{2}}\norm{\Bar{F_1}_t^{1/2} - \Bar{F_2}_t^{1/2}}_F \xrightarrow[]{\mathcal{D}} 0$.
\end{proof}

\begin{theorem2}
Let $X_A$ be the dataset for the task $T_A$ with the objective function $L_A$, and $X_B$ be the dataset for the task $T_B$ with the objective function $L_B$. Assume $X_A$ and $X_B$ have the same distribution. Consider an $\varepsilon$-approximation network $N$ trained using both datasets $X_A^{(1)}$ and $X_B^{(1)}$ respectively with the objective functions $L_A$ and $L_B$ to result weights $\theta_{A_{t}}$ and $\theta_{B_{t}}$ at time $t$. Under the same assumptions on the moment of gradient noise in SGD algorithm and the loss function stated in Theorem~\ref{theorem1}, the FTD from the task $A$ to the task $B$ computed from the Fisher Information matrices of the average of estimated weights up to the current time $t$ converges to a constant as $t \rightarrow \infty$. That is,
\begin{align}
    d_t = \frac{1}{\sqrt{2}}\norm{\Bar{F_A}_t^{1/2} - \Bar{F_B}_t^{1/2}}_F \xrightarrow[]{\mathcal{D}} \frac{1}{\sqrt{2}}\norm{{F_A^*}^{1/2} - {F_B^*}^{1/2}}_F,
\end{align}
where $\Bar{F}_{A_{t}}$ is given by $\Bar{F}_{A_{t}} = F(\bar{\theta}_{A_{t}})$ with $\bar{\theta}_{A_{t}} = \frac{1}{t}\sum_t \theta_{A_{t}}$, and $\Bar{F}_{B_{t}}$ is defined in a similar way.
\end{theorem2}

\begin{proof}[\textbf{Proof of Theorem \ref{theorem2}}]
Let $\theta_{A_{t}}$ and $\theta_{B_{t}}$ be the sets of weights at time $t$ from the $\varepsilon$-approximation network $N$ trained using both data sets $X_A^{(1)}$ and $X_B^{(1)}$, respectively with the objective functions $L_A$ and $L_B$. Since the objective functions are strongly convex, both of these sets of weights will obtain the optimum solutions $\theta_A^*$ and $\theta_B^*$ after training a certain number of epochs with stochastic gradient descend. Similar to the proof of Theorem \ref{theorem1}, by invoking the Polyak et al. \cite{Polyak1992AccelerationOS}, random vectors of $\sqrt{t}(\Bar{\theta_A}_t - \theta^*_A)$ and $\sqrt{t}(\Bar{\theta_B}_t - \theta^*_B)$ are asymptotically normal:
\begin{align}
    \sqrt{t}(\Bar{\theta_A}_t - \theta^*_A) \xrightarrow[]{\mathcal{D}} \mathcal{N}(0, \Sigma_A), \label{eq17}\\
    \sqrt{t}(\Bar{\theta_B}_t - \theta^*_B) \xrightarrow[]{\mathcal{D}} \mathcal{N}(0, \Sigma_B), \label{eq18}
\end{align}
where $\Sigma_A = \mathbf{H}\big(L(\theta^*_A)\big)^{-1} C_A \mathbf{H}^T\big(L(\theta^*_A)\big)^{-1}$, and $\Sigma_B = \mathbf{H}\big(L(\theta^*_B)\big)^{-1} C_B \mathbf{H}^T\big(L(\theta^*_B)\big)^{-1}$ (here, $C_A$ and $C_B$ denote the conditional covariance matrices, corresponding to the gradient noise in $\theta^*_A$ and $\theta^*_A$, respectively).  The Fisher Information $F(\theta)$ is a continuous and differentiable function of $\theta$, and it is also a positive definite matrix; thus, $F(\theta)^{1/2}$ is well-defined. Now, by applying the Delta method to Equation (\ref{eq17}), we have: 
\begin{equation}\label{eq20}
    (\Bar{F_A}_t^{1/2} - {F_A^*}^{1/2}) \xrightarrow[]{\mathcal{D}} \mathcal{N} \big( 0, \frac{1}{t}\Sigma_A^* \big),
\end{equation}
where $\bar{F}_{At} = F(\bar{\theta}_{At})$, and the covariance matrix is given by $\Sigma_A^* = \mathbf{J}_{\theta} \Big( \mathbf{vec} \big( F(\theta_A^*)^{1/2} \big) \Big) \Sigma_A \mathbf{J}_{\theta} \Big( \mathbf{vec} \big( F(\theta_A^*)^{1/2} \big) \Big)^T$. Likewise, from Equation (\ref{eq18}), we have:
\begin{equation}\label{eq21}
    (\Bar{F_B}_t^{1/2} - {F_B^*}^{1/2}) \xrightarrow[]{\mathcal{D}} \mathcal{N} \big( 0, \frac{1}{t}\Sigma_B^* \big),
\end{equation}
where $\bar{F}_{Bt} = F(\bar{\theta}_{Bt})$, and the covariance matrix is given by $\Sigma_B^* = \mathbf{J}_{\theta}\Big(\mathbf{vec}\big(F(\theta_B^*)^{1/2}\big)\Big) \mathbf{J}_{\theta}\Big(\mathbf{vec}\big(F(\theta_B^*)^{1/2}\big)\Big)^T$. From Equation (\ref{eq20}) and (\ref{eq21}), we obtain:
\begin{equation}
    (\Bar{F_A}_t^{1/2} - \Bar{F_B}_t^{1/2}) \xrightarrow[]{\mathcal{D}} \mathcal{N}\Big(\mu_2, V_2\Big),
\end{equation}
where $\mu_2=({F_A^*}^{1/2} - {F_B^*}^{1/2})$ and $V_2 = \frac{1}{t}(\Sigma_A^*+\Sigma_B^*)$. Since $(\Bar{F_A}_t^{1/2}-\Bar{F_B}_t^{1/2}) - ({F_A^*}^{1/2}-{F_B^*}^{1/2})$ is asymptotically normal with the covariance goes to zero as $t$ approaches infinity, all of the entries go to zero, we conclude that 
\begin{equation}
    d_t = \frac{1}{\sqrt{2}}\norm{\Bar{F_A}_t^{1/2} - \Bar{F_B}_t^{1/2}}_F \xrightarrow[]{\mathcal{D}} \frac{1}{\sqrt{2}}\norm{{F_A^*}^{1/2} - {F_B^*}^{1/2}}_F.
\end{equation}
\end{proof}

\bibliography{ref}
\bibliographystyle{ieeetr}

\vfill
\pagebreak
\begin{IEEEbiography}[{\includegraphics[width=1in,height=1.25in,clip,keepaspectratio]{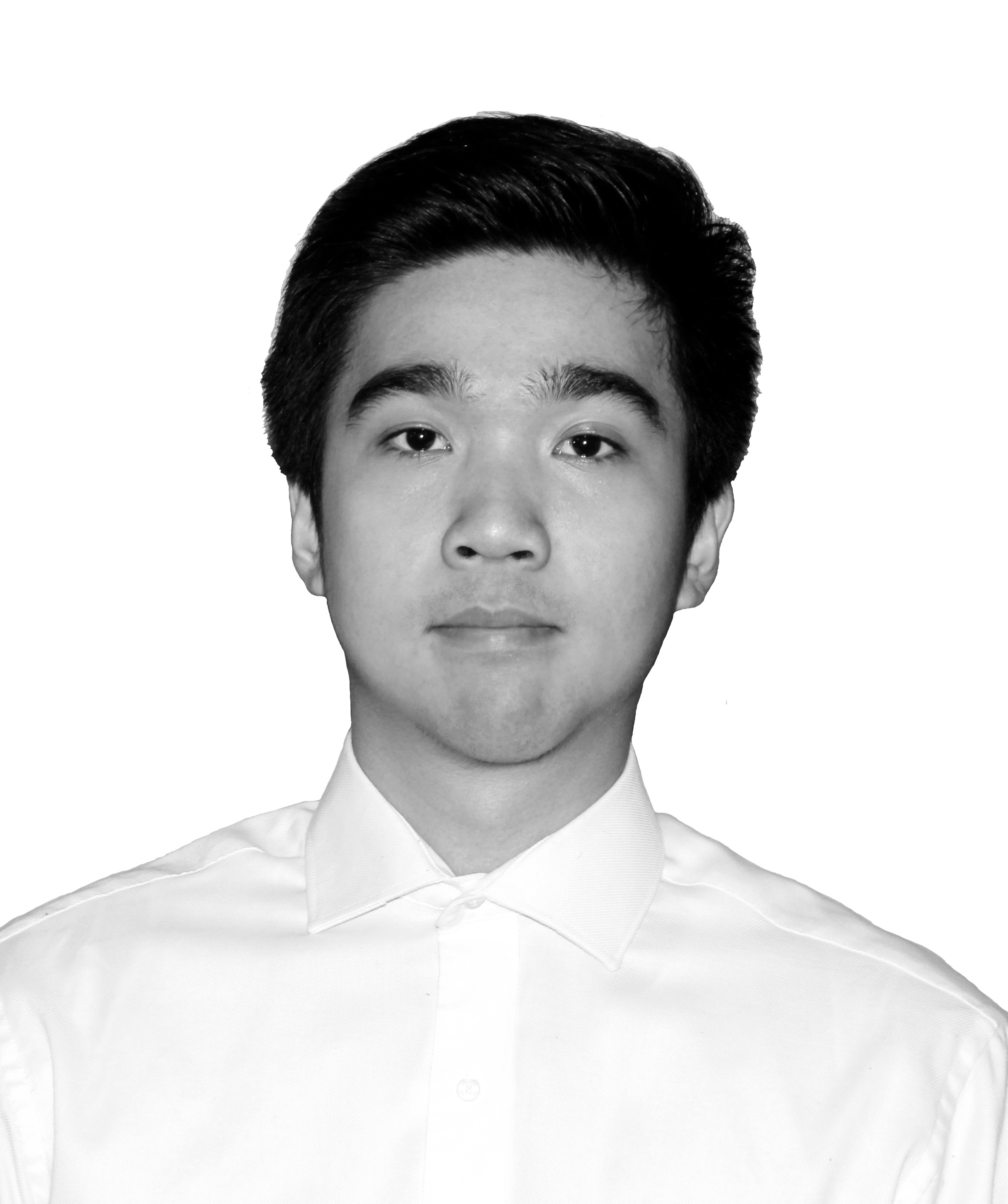}}]{Cat P. Le} received the B.S. degree with Summa Cum Laude in electrical and computer engineering from Rutgers University, in 2016 and the M.S. degree in electrical engineering from California Institute of Technology (Caltech), in 2017. He is currently pursuing the Ph.D. degree in electrical and computer engineering at Duke University, under the supervision of Dr. Vahid Tarokh. His research interest includes image processing, computer vision, machine learning, with a focus on transfer learning, continual learning, and neural architecture search. His awards and honors include the Matthew Leydt Award, John B. Smith Award.
\end{IEEEbiography}

\begin{IEEEbiography}[{\includegraphics[width=1in,height=1.25in,clip,keepaspectratio]{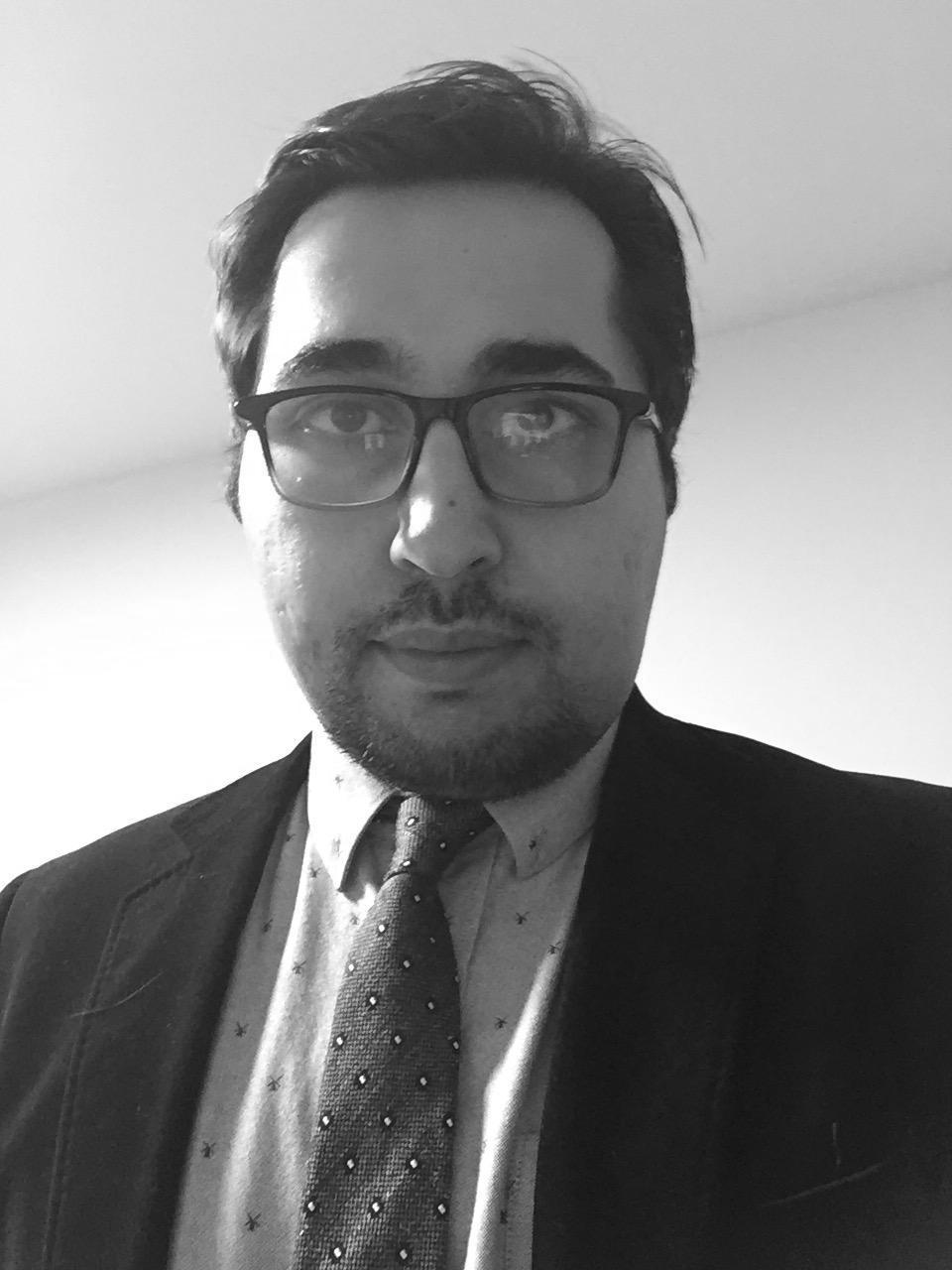}}]{Mohammadreza Soltani} is currently a postdoctoral associate in the Department of Electrical and Computer Engineering at Duke University. He received his Ph.D. degree from Iowa State University in 2019 in Electrical Engineering. He has two master's degrees in Electrical Engineering and Telecommunication Engineering with a minor in Mathematics. Mohammadreza's research interest lies in the intersection of signal processing, machine learning, and numerical optimization. His recent projects include neural architecture search, radar signal processing using machine learning techniques, meta-material design using deep learning. 
\end{IEEEbiography}

\begin{IEEEbiography}[{\includegraphics[width=1in,height=1.25in,clip,keepaspectratio]{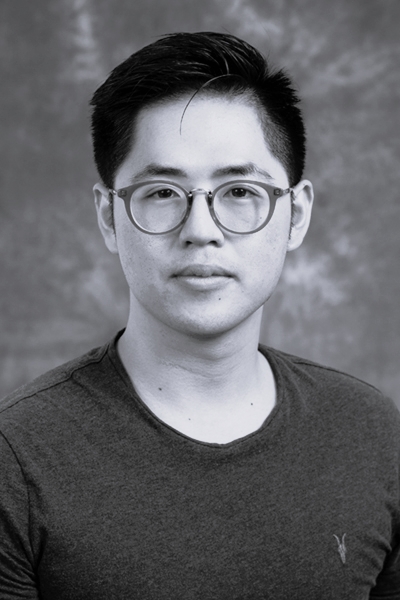}}]{Juncheng Dong} is a master student of Computer Science at Duke University at where he received the Dean's Research Award. Under supervision of Prof.Vahid Tarokh, his research interest includes machine learning, representation learning, reinforcement learning, etc. Before joining Duke University, he studied Computer Science and Mathematics at University of California - San Diego(UCSD). Before UCSD, he graduated from NanYang Model High School at Shanghai, China.
\end{IEEEbiography}

\begin{IEEEbiography}[{\includegraphics[width=1in,height=1.25in,clip,keepaspectratio]{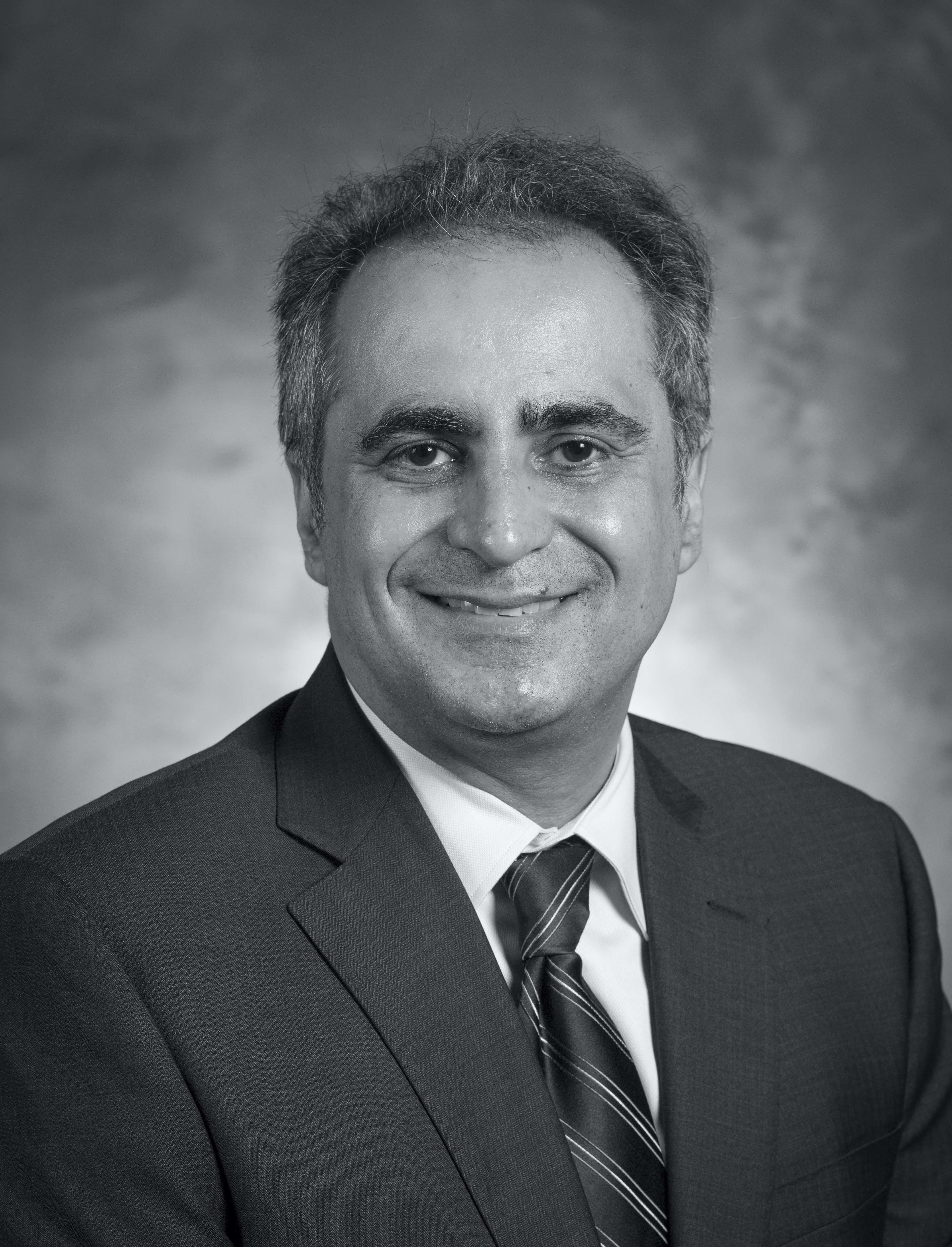}}]{Vahid Tarokh} worked at AT\&T Labs-Research until 2000. From 2000-2002, he was an Associate Professor at Massachusetts Institute of Technology (MIT). In 2002, he joined Harvard University as a Hammond Vinton Hayes Senior Fellow of Electrical Engineering and Perkins Professor of Applied Mathematics. He joined joined Duke University in Jan 2018, as the Rhodes Family Professor of Electrical and Computer Engineering, Computer Science, and Mathematics and Bass Connections Endowed Professor. He was also a Gordon Moore Distinguished Research Fellow at Caltech in 2018. Since Jan 2019, he has also been named as a Microsoft Data Science Investigator at Duke University.
\end{IEEEbiography}

\EOD

\end{document}